%% file: main.tex
\documentclass{article} 

\usepackage[utf8]{inputenc} 
\usepackage[T1]{fontenc}    
\usepackage{hyperref}       
\usepackage{url}            
\usepackage{booktabs}       
\usepackage{amsfonts}       
\usepackage{nicefrac}       
\usepackage{microtype}      
\usepackage{graphicx}
\usepackage{subcaption}

\usepackage{hyperref}

\usepackage[accepted]{icml2020}

\icmltitlerunning{\METHOD: State-Covering Self-Supervised Reinforcement Learning}

\input{math_commands.tex}

\input{paper_specific_commands_and_packages.tex}

 \PassOptionsToPackage{numbers, compress}{natbib}

\begin{document}
\twocolumn[
\icmltitle{\METHOD: State-Covering Self-Supervised\\Reinforcement Learning}


\icmlsetsymbol{equal}{*}

\begin{icmlauthorlist}
\icmlauthor{Vitchyr H. Pong}{equal,berk}
\icmlauthor{Murtaza Dalal}{equal,berk}
\icmlauthor{Steven Lin}{equal,berk}
\icmlauthor{Ashvin Nair}{berk}
\icmlauthor{Shikhar Bahl}{berk}
\icmlauthor{Sergey Levine}{berk}
\end{icmlauthorlist}

\icmlaffiliation{berk}{University of California, Berkeley}

\icmlcorrespondingauthor{Vitchyr H. Pong}{vitchyr@eecs.berkeley.edu}

\icmlkeywords{deep reinforcement learning, goal-conditioned reinforcement learning, goals, exploration, goal-directed exploration}

\vskip 0.3in
]
\printAffiliationsAndNotice{\icmlEqualContribution}

\begin{abstract}
Autonomous agents that must exhibit flexible and broad capabilities will need to be equipped with large repertoires of skills.
Defining each skill with a manually-designed reward function limits this repertoire and imposes a manual engineering burden.
Self-supervised agents that set their own goals can automate this process, but designing appropriate goal setting objectives can be difficult, and often involves heuristic design decisions.
In this paper, we propose a formal exploration objective for goal-reaching policies that maximizes state coverage.
We show that this objective is equivalent to maximizing goal reaching performance together with the entropy of the goal distribution, where goals correspond to full state observations.
To instantiate this principle, we present an algorithm called \METHOD for learning a maximum-entropy goal distributions and prove that, under regularity conditions, \METHOD converges to a uniform distribution over the set of valid states, even when we do not know this set beforehand.
Our experiments show that combining \METHOD for learning goal distributions with existing goal-reaching methods outperforms a variety of prior methods on open-sourced visual goal-reaching tasks and that \METHOD enables a real-world robot to learn to open a door, entirely from scratch, from pixels, and without any manually-designed reward function.
\end{abstract}

\section{Introduction}\label{sec:introduction}
\input{introduction.tex}

\setlength{\textfloatsep}{0.5\baselineskip plus 0.2\baselineskip minus 0.2\baselineskip}
\setlength{\dbltextfloatsep}{0.5\baselineskip plus 0.2\baselineskip minus 0.2\baselineskip}

\section{Problem Formulation}\label{sec:background}
\input{preliminaries.tex}

\section{\METHOD: Learning a Maximum Entropy Goal Distribution}
\label{sec:method}
\input{method.tex}

\section{Training Goal-Conditioned Policies with \METHOD}
\label{sec:train-policy}
\input{training_policy.tex}

\section{Related Work}\label{sec:related_work}
\input{related_work.tex}

\section{Experiments}\label{sec:experiments}
\input{experiments.tex}

\section{Conclusion}\label{sec:conclusion}
\input{conclusion.tex}

\section{Acknowledgement}
This research was supported by Berkeley DeepDrive, Huawei, ARL DCIST CRA W911NF-17-2-0181, NSF IIS-1651843, and the Office of Naval Research, as well as Amazon, Google, and NVIDIA.
We thank Aviral Kumar, Carlos Florensa, Aurick Zhou, Nilesh Tripuraneni, Vickie Ye, Dibya Ghosh, Coline Devin, Rowan McAllister, John D. Co-Reyes, various members of the Berkeley Robotic AI \& Learning (RAIL) lab, and anonymous reviewers for their insightful discussions and feedback.

{\small
\bibliographystyle{corlabbrvnat}
\bibliography{main.bib}
}

\clearpage
\newpage
\input{appendix.tex}

\end{document}

%% file: math_commands.tex
\usepackage{amsmath,amsfonts,bm}
\usepackage{mathtools}

\def\Figref#1{Figure~\ref{#1}}

\def\eqref#1{equation~\ref{#1}}
\def\Eqref#1{Equation~\ref{#1}}

\def\1{\bm{1}}

\DeclareMathAlphabet{\mathsfit}{\encodingdefault}{\sfdefault}{m}{sl}
\SetMathAlphabet{\mathsfit}{bold}{\encodingdefault}{\sfdefault}{bx}{n}

\def\gF{{\mathcal{F}}}

\def\gH{{\mathcal{H}}}

\def\gQ{{\mathcal{Q}}}

\def\gX{{\mathcal{X}}}

\newcommand{\E}{\mathbb{E}}

\newcommand{\R}{\mathbb{R}}

\newcommand{\KL}{D_{\mathrm{KL}}}

\DeclarePairedDelimiterX{\infdivx}[2]{(}{)}{%
  #1\;\delimsize\|\;#2%
}
\newcommand{\kld}{\KL\infdivx}

\newcommand{\Var}{\mathrm{Var}}

\newcommand{\Cov}{\mathrm{Cov}}


%% file: paper_specific_commands_and_packages.tex
\usepackage{xspace}
\usepackage[utf8]{inputenc}
\usepackage[english]{babel}
\usepackage{amsthm}
\usepackage{amssymb}
\usepackage{mathtools}
\usepackage{hyperref}
\usepackage{appendix}
\usepackage[normalem]{ulem}

\usepackage{etoolbox}
\makeatletter
\patchcmd{\hyper@makecurrent}{%
    \ifx\Hy@param\Hy@chapterstring
        \let\Hy@param\Hy@chapapp
    \fi
}{%
    \iftoggle{inappendix}{
        \@checkappendixparam{chapter}%
        \@checkappendixparam{section}%
        \@checkappendixparam{subsection}%
        \@checkappendixparam{subsubsection}%
        \@checkappendixparam{paragraph}%
        \@checkappendixparam{subparagraph}%
    }{}%
}{}{\errmessage{failed to patch}}

\newcommand*{\@checkappendixparam}[1]{%
    \def\@checkappendixparamtmp{#1}%
    \ifx\Hy@param\@checkappendixparamtmp
        \let\Hy@param\Hy@appendixstring
    \fi
}
\makeatletter

\newtoggle{inappendix}
\togglefalse{inappendix}

\apptocmd{\appendix}{\toggletrue{inappendix}}{}{\errmessage{failed to patch}}
\apptocmd{\subappendices}{\toggletrue{inappendix}}{}{\errmessage{failed to patch}}

\newcommand\strike{\bgroup\markoverwith{\textcolor{red}{\rule[0.5ex]{2pt}{0.4pt}}}\ULon}
\definecolor{darkgreen}{RGB}{0,180,56}

\newcommand{\fcaption}[1]{\caption{{\footnotesize{#1}}}}

\newcommand{\dent}{{d_\gH}}

\newcommand{\limt}{\lim_{t \rightarrow \infty}}

\newcommand{\pap}{{\theta}}
\newcommand{\papprox}{{p_\pap}}

\newcommand{\pspgt}{p(\SF \mid \pgt)}

\newcommand{\wt}{{w_{t, \alpha}}}

\newcommand{\FullSpaceDim}{n}
\newcommand{\FullSpace}{{\mathbb{R}^\FullSpaceDim}}
\newcommand{\Imgs}{{\mathcal{S}}}

\newcommand{\Ss}{{\mathcal{S}}}
\newcommand{\As}{{\mathcal{A}}}
\newcommand{\Gs}{{\mathcal{G}}}

\newcommand{\support}{{\text{support}}}

\newcommand{\G}{\mathbf{G}}
\newcommand{\g}{\mathbf{g}}
\newcommand{\SF}{{\mathbf{S}}}

\newcommand{\St}{\SF}
\newcommand{\st}{\mathbf{s}}

\newcommand{\at}{\mathbf{a}}
\newcommand{\dyn}{p(\st_{t+1} \mid \st_t, \at_t)}

\mathchardef\mhyphen="2D

\newcommand{\HG}{\gH(\G)}
\newcommand{\HGS}{\gH(\G \mid \SF)}

\newcommand{\Loss}{\mathcal{L}}

\newcommand{\pstate}{{p^S_{\pgparam}}}
\newcommand{\pstatet}{{p^S_{\pgparam_t}}}

\newcommand{\U}{{U_\Imgs}}

\newcommand{\pgparam}{{\phi}}

\newcommand{\pgparamtt}{{\pgparam_{t+1}}}
\newcommand{\pg}{{q^G_{\pgparam}}}
\newcommand{\pgt}{{q^G_{\pgparam_t}}}
\newcommand{\pgtt}{{q^G_{\pgparam_{t+1}}}}

\newcommand{\pet}{\pstatet}

\newcommand{\pskewed}{{p_{\text{skewed}}}}
\newcommand{\pskewedt}{{p_{\text{skewed}_t}}}

\newcommand{\pskewedtt}{{p_{\text{skewed}_{t+1}}}}
\newcommand{\pskewedn}{{p_{\text{skewed}_n}}}
\newcommand{\pskewednn}{{p_{\text{skewed}_{n+1}}}}

\newcommand{\ptt}{{p_\theta^t}}
\newcommand{\htt}{{\gH_\theta^t}}

\newcommand{\METHOD}{Skew-Fit\xspace}

\newtheorem*{rep@theorem}{\rep@title}
\newcommand{\newreptheorem}[2]{%
\newenvironment{rep#1}[1]{%
 \def\rep@title{#2 \ref{##1}}%
 \begin{rep@theorem}}%
 {\end{rep@theorem}}}
\makeatother
\newreptheorem{theorem}{Theorem}
\newreptheorem{lemma}{Lemma}
\newreptheorem{assumption}{Assumption}

\newtheorem{theorem}{Theorem}[section]

\newtheorem{lemma}[theorem]{Lemma}

%% file: introduction.tex
\setlength{\intextsep}{-.9pt}
\begin{figure}
  \includegraphics[width=\linewidth]{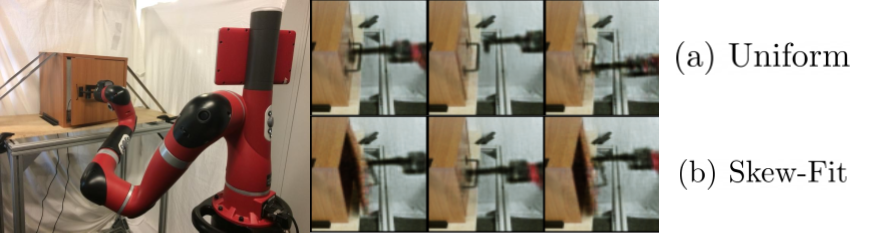}
  \fcaption{
  Left: Robot learning to open a door with \METHOD, without any task reward.
  Right: Samples from a goal distribution when using (a) uniform and (b) \METHOD sampling.
  When used as goals, the diverse samples from \METHOD encourage the robot to practice opening the door more frequently.
  }
  \label{fig:offline-sk-real}
\end{figure}
Reinforcement learning (RL) provides an appealing formalism for automated learning of behavioral skills, but separately learning every potentially useful skill becomes prohibitively time consuming, both in terms of the experience required for the agent and the effort required for the user to design reward functions for each behavior.
What if we could instead design an unsupervised RL algorithm that automatically explores the environment and iteratively distills this experience into general-purpose policies that can accomplish new user-specified tasks at test time?

In the absence of any prior knowledge, an effective exploration scheme is one that visits as many states as possible, allowing a policy to autonomously prepare for user-specified tasks that it might see at test time.
We can formalize the problem of visiting as many states as possible as one of maximizing the \emph{state entropy} $\gH(\SF)$ under the current policy.\stepcounter{footnote}\footnote{We consider the distribution over terminal states in a finite horizon task  and believe this work can be extended to infinite horizon stationary distributions.}
Unfortunately, optimizing this objective alone does not result in a policy that can solve new tasks: it only knows how to maximize state entropy.
In other words, to develop principled unsupervised RL algorithms that result in useful policies, maximizing $\gH(\SF)$ is not enough.
We need a mechanism that allows us to reuse the resulting policy to achieve new tasks at test-time.

We argue that this can be accomplished by performing \textit{goal-directed exploration}:
a policy should autonomously visit as many states as possible, but after autonomous exploration, a user should be able to reuse this policy by giving it a goal $\G$ that corresponds to a state that it must reach.
While not all test-time tasks can be expressed as reaching a goal state, a wide range of tasks can be represented in this way.
Mathematically, the goal-conditioned policy should minimize the conditional entropy over the states given a goal, $\gH(\SF \mid \G)$, so that there is little uncertainty over its state given a commanded goal.
This objective provides us with a principled way to train a policy to explore all states (maximize $\gH(\SF)$) such that the state that is reached can be determined by commanding goals (minimize $\gH(\SF \mid \G)$).

Directly optimizing this objective is in general intractable, since it requires optimizing the entropy of the marginal state distribution, $\gH(\SF)$.
However, we can sidestep this issue by noting that the objective is the mutual information between the state and the goal, $I(\SF; \G)$, which can be written as:
{
  \setlength{\abovedisplayskip}{18pt}%
  \setlength{\belowdisplayskip}{0pt}%
  \setlength{\abovedisplayshortskip}{18pt}%
  \setlength{\belowdisplayshortskip}{0pt}
\begin{align}
    \label{eq:hg-hgs}
    \gH(\SF) - \gH(\SF|\G)
    =
    I(\SF; \G)
    =
    \gH(\G) - \gH(\G|\SF).
\end{align}
}

\Eqref{eq:hg-hgs} thus gives an equivalent objective for an unsupervised RL algorithm:
the agent should set diverse goals, maximizing $\gH(\G)$, and learn how to reach them, minimizing $\gH(\G \mid \SF)$.

While learning to reach goals is the typical objective studied in goal-conditioned RL~\citep{kaelbling1993goals,andrychowicz2017her}, setting goals that have maximum diversity is crucial for effectively learning to reach all possible states.
Acquiring such a maximum-entropy goal distribution is challenging in environments with complex, high-dimensional state spaces, where even knowing which states are valid presents a major challenge.
For example, in image-based domains, a uniform goal distribution requires sampling uniformly from the set of realistic images, which in general is unknown a priori.

Our paper makes the following contributions.
First, we propose a principled objective for unsupervised RL, based on \autoref{eq:hg-hgs}.
While a number of prior works ignore the $\gH(\G)$ term, we argue that jointly optimizing the entire quantity is needed to develop effective exploration.
Second, we present a general algorithm called \METHOD and prove that under regularity conditions \METHOD learns a sequence of generative models that converges to a uniform distribution over the goal space, even when the set of valid states is unknown (e.g., as in the case of images).
Third, we describe a concrete implementation of \METHOD and empirically demonstrate that this method achieves state of the art results compared to a large number of prior methods for goal reaching with visually indicated goals, including a real-world manipulation task, which requires a robot to learn to open a door from scratch in about five hours, directly from images, and without any manually-designed reward function.

%% file: preliminaries.tex
To ensure that an unsupervised reinforcement learning agent learns to reach all possible states in a controllable way, we maximize the mutual information between the state $\SF$ and the goal $\G$, $I(\SF; \G)$, as stated in \Eqref{eq:hg-hgs}. This section discusses how to optimize \Eqref{eq:hg-hgs} by splitting the optimization into two parts: minimizing $\gH(\G \mid \SF)$ and maximizing $\gH(\G)$.

\subsection{Minimizing $\gH(\G \mid \SF)$: Goal-Conditioned Reinforcement Learning}
Standard RL considers a Markov decision process (MDP), which has a state space $\Ss$, action space $\As$, and unknown dynamics $\dyn: \Ss \times \Ss \times \As \mapsto [0, +\infty)$.
Goal-conditioned RL also includes a goal space $\Gs$. For simplicity, we will assume in our derivation that the goal space matches the state space, such that $\Gs = \Ss$, though the approach extends trivially to the case where $\Gs$ is a hand-specified subset of $\Ss$, such as the global XY position of a robot.
A goal-conditioned policy $\pi(\at \mid \st, \g)$ maps a state $\st \in \Ss$ and goal $\g \in \Ss$ to a distribution over actions $\at \in \As$, and its objective is to reach the goal, i.e., to make the current state equal to the goal.

Goal-reaching can be formulated as minimizing $\gH(\G \mid \SF)$, and many practical goal-reaching algorithms~\citep{kaelbling1993goals,lillicrap2015continuous, schaul2015uva, andrychowicz2017her, nair2018rig, pong2018tdm,florensa2018selfsupervised} can be viewed as approximations to this objective by observing that the optimal goal-conditioned policy will deterministically reach the goal, resulting in a conditional entropy of zero: $\gH(\G \mid \SF) = 0$.
See \autoref{sec:analysis-appendix} for more details.
Our method may thus be used in conjunction with any of these prior goal-conditioned RL methods in order to jointly minimize $\gH(\G \mid \SF)$ and maximize $\gH(\G)$.

\subsection{Maximizing $\gH(\G)$: Setting Diverse Goals}\label{sec:prelim-max-ent}
We now turn to the problem of setting diverse goals or, mathematically, maximizing the entropy of the goal distribution $\gH(\G)$.
Let $\U$ be the uniform distribution over $\Imgs$, where we assume $\Imgs$ has finite volume so that the uniform distribution is well-defined.
Let $\pg$ be the goal distribution from which goals $\G$ are sampled, parameterized by $\pgparam$.
Our goal is to maximize the entropy of $\pg$, which we write as $\gH(\G)$.
Since the maximum entropy distribution over $\Imgs$ is the uniform distribution $\U$, maximizing $\gH(\G)$ may seem as simple as choosing the uniform distribution to be our goal distribution: $\pg = \U$.
However, this requires knowing the uniform distribution over valid states, which may be difficult to obtain when $\Imgs$ is a subset of $\FullSpace$, for some $\FullSpaceDim$.
For example, if the states correspond to images viewed through a robot's camera, $\Imgs$ corresponds to the (unknown) set of valid images of the robot's environment, while $\FullSpace$ corresponds to all possible arrays of pixel values of a particular size.
In such environments, sampling from the uniform distribution $\FullSpace$ is unlikely to correspond to a valid image of the real world.
Sampling uniformly from $\Imgs$ would require knowing the set of all possible valid images, which we assume the agent does not know when starting to explore the environment.

While we cannot sample arbitrary states from $\Imgs$, we can sample states by performing goal-directed exploration.
To derive and analyze our method, we introduce a simple model of this process:
a goal $\G \sim \pg$ is sampled from the goal distribution $\pg$, and
then the goal-conditioned policy $\pi$ attempts to achieve this goal, which results in a distribution of terminal states $\SF \in \Imgs$.
We abstract this entire process by writing the resulting marginal distribution over $\SF$ as \mbox{$\pstate (\SF) \triangleq \int_\mathcal{G} \pg(\G) p(\SF \mid \G) d\G$}, where the subscript $\pgparam$ indicates that the marginal $\pstate$ depends indirectly on $\pg$ via the goal-conditioned policy $\pi$.
We assume that $\pstate$ has full support, which can be accomplished with an epsilon-greedy goal reaching policy in a communicating MDP.
We also assume that the entropy of the resulting state distribution $\gH(\pstate)$ is no less than the entropy of the goal distribution $\gH(\pg)$.
Without this assumption, a policy could ignore the goal and stay in a single state, no matter how diverse and realistic the goals are.
\footnote{
Note that this assumption does \textbf{not} require that the entropy of $\pstate$ is strictly larger than the entropy of the goal distribution, $\pg$.
}
This simplified model allows us to analyze the behavior of our goal-setting scheme separately from any specific goal-reaching algorithm. We will however show in Section~\ref{sec:experiments} that we can instantiate this approach into a practical algorithm that jointly learns the goal-reaching policy. In summary, our goal is to acquire a maximum-entropy goal distribution $\pg$ over valid states $\Imgs$, while only having access to state samples from $\pstate$.

%% file: method.tex
Our method, \METHOD, learns a maximum entropy goal distribution $\pg$ using samples collected from a goal-conditioned policy.
We analyze the algorithm and show that \METHOD maximizes the goal distribution entropy, and present a practical instantiation for unsupervised deep RL.

\subsection{\METHOD Algorithm}\label{sec:method-description}
To learn a uniform distribution over \emph{valid} goal states, we present a method that iteratively increases the entropy of a generative model $\pg$.
In particular, given a generative model $\pgt$ at iteration $t$, we want to train a new generative model, $\pgtt$ that has higher entropy.
While we do not know the set of valid states $\Imgs$, we could sample states \mbox{$\st_n \overset{\text{iid}}{\sim} \pstatet$} using the goal-conditioned policy,
and use the samples to train $\pgtt$.
However, there is no guarantee that this would increase the entropy of $\pgtt$.

\begin{figure}[ht]
    \centering
    \includegraphics[width=\linewidth]{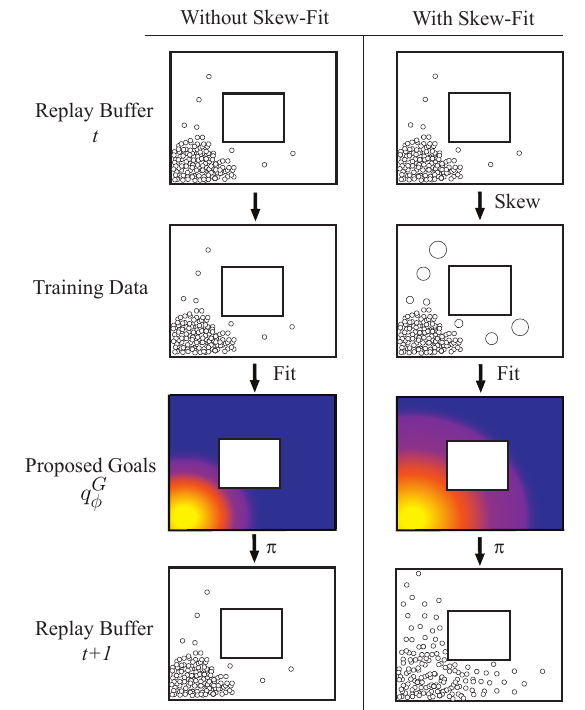}
    \fcaption{Our method, \METHOD, samples goals for goal-conditioned RL.
    We sample states from our replay buffer, and give more weight to rare states.
    We then train a generative model $\pg_{t+1}$ with the weighted samples.
    By sampling new states with goals proposed from this new generative model, we obtain a higher entropy state distribution in the next iteration.}
    \label{fig:main-fig}
\end{figure}

The intuition behind our method is simple: rather than fitting a generative model to these samples $\st_n$, we \textit{skew} the samples so that rarely visited states are given more weight.
See \Figref{fig:main-fig} for a visualization of this process.
How should we skew the samples if we want to maximize the entropy of $\pgtt$?
If we had access to the density of each state, $\pet(\St)$, then we could simply weight each state by $1/\pet(\St)$.
We could then perform maximum likelihood estimation (MLE) for the uniform distribution by using the following importance sampling (IS) loss to train $\pgparamtt$:
\begin{align}
\Loss(\pgparam)\nonumber
    &= \E_{\St \sim \U} \left[ \log \pg(\St)\right]
\\\nonumber
    &= \E_{\St \sim \pet}\left[ \frac{\U(\St)}{\pet(\St)} \log \pg(\St)\right]
\\\nonumber
    &\propto \E_{\St \sim \pet}\left[ \frac{1}{\pet(\St)}\log \pg(\St)\right]
\nonumber
\end{align}
where we use the fact that the uniform distribution $\U(\St)$ has constant density for all states in $\Imgs$.
However, computing this density $\pet(\St)$ requires marginalizing out the MDP dynamics, which requires an accurate model of both the dynamics and the goal-conditioned policy.

We avoid needing to model the entire MDP process by approximating $\pet(\St)$ with our previous learned generative model: \mbox{$\pstatet(\St) \approx \pgt(\St)$}.
We therefore weight each state by the following weight function
\begin{align}\label{eq:weight-defn}
    \wt(\SF) \triangleq \pgt(\SF)^\alpha, \quad \alpha < 0.
\end{align}
where $\alpha$ is a hyperparameter that controls how heavily we weight each state.
If our approximation $\pgt$ is exact, we can choose $\alpha = -1$ and recover the exact IS procedure described above.
If $\alpha = 0$, then this skew step has no effect.
By choosing intermediate values of $\alpha$, we trade off the reliability of our estimate $\pgt(\St)$ with the speed at which we want to increase the goal distribution entropy.

\paragraph{Variance Reduction}
As described, this procedure relies on IS, which can have high variance, particularly if $\pgt(\St) \approx 0$.
We therefore choose a class of generative models where the probabilities are prevented from collapsing to zero, as we will describe in \autoref{sec:train-policy} where we provide generative model details.
To further reduce the variance, we train $\pgtt$ with sampling importance resampling (SIR)~\citep{rubin1988using} rather than IS.
Rather than sampling from $\pet$ and weighting the update from each sample by $\wt$, SIR explicitly defines a skewed empirical distribution as
\begin{align}\label{eq:pskew-defn}
    \pskewedt(\st) \triangleq \frac{1}{Z_\alpha} \wt(\st) \delta(\st \in \{\st_n\}_{n=1}^{N})
    \\\nonumber
    Z_\alpha = \sum_{n=1}^N \wt(\st_n),\ \st_n \overset{\text{iid}}{\sim} \pstatet,
\end{align}
where $\delta$ is the indicator function and $Z_\alpha$ is the normalizing coefficient.
We note that computing $Z_\alpha$ adds little computational overhead, since all of the weights already need to be computed.
We then fit the generative model at the next iteration $\pgtt$ to $\pskewedt$ using standard MLE.
We found that using SIR resulted in significantly lower variance than IS.
See \autoref{sec:analysis-variance} for this comparision.

\paragraph{Goal Sampling Alternative}
Because $\pgtt \approx \pskewedt$, at iteration $t+1$, one can sample goals from either $\pgtt$ or $\pskewedt$.
Sampling goals from $\pskewedt$ may be preferred if sampling from the learned generative model $\pgtt$ is computationally or otherwise challenging.
In either case, one still needs to train the generative model $\pgt$ to create $\pskewedt$.
In our experiments, we found that both methods perform well.

\paragraph{Summary}
Overall, \METHOD collects states from the environment and resamples each state in proportion to \autoref{eq:weight-defn} so that low-density states are resampled more often.
\METHOD is shown in \Figref{fig:main-fig} and summarized in Algorithm \ref{alg:method}.
We now discuss conditions under which \METHOD converges to the uniform distribution.

\vspace{0.1in}
\begin{algorithm}
   	\fcaption{\METHOD}
   	\label{alg:method}
   	\begin{algorithmic}[1]
   	\FOR{Iteration $t=1, 2, ...$}
        \STATE Collect $N$ states $\{\st_n\}_{n=1}^N$ by sampling goals from $\pgt$ (or $\pskewed_{t-1}$) and running goal-conditioned policy.
        \STATE Construct skewed distribution $\pskewedt$ (\Eqref{eq:weight-defn} and \Eqref{eq:pskew-defn}).
        \STATE Fit $\pgtt$ to skewed distribution $\pskewedt$ using MLE.
   	\ENDFOR
   	\end{algorithmic}
\end{algorithm}

\subsection{\METHOD Analysis}\label{sec:analysis}
This section provides conditions under which $\pgt$ converges in the limit to the uniform distribution over the state space $\Imgs$.
We consider the case where $N \rightarrow \infty$, which allows us to study the limit behavior of the goal distribution $\pskewedt$.
Our most general result is stated as follows:
\begin{lemma}\label{lemma:general-convergence}
Let $\Imgs$ be a compact set.
Define the set of distributions $\gQ = \{p : \support(p) \subseteq \Imgs\}$.
Let $\gF: \gQ \mapsto \gQ$ be continuous with respect to the pseudometric \mbox{$\dent(p, q) \triangleq |\gH(p) - \gH(q)|$} and $\gH(\gF(p)) \geq \gH(p)$ with equality if and only if $p$ is the uniform probability distribution on $\Imgs$, denoted as $\U$.
Define the sequence of distributions $P = (p_1, p_2, \dots)$ by starting with any $p_1 \in \gQ$ and recursively defining $p_{t+1} = \gF(p_t)$.
The sequence $P$ converges to $\U$ with respect to $\dent$. In other words, \mbox{$\lim_{t \rightarrow 0} |\gH(p_t) - \gH(\U)| \rightarrow 0$}.
\end{lemma}
\begin{proof}
See Appendix Section \ref{sec:general-proof}.
\end{proof}

We will apply Lemma \ref{lemma:general-convergence} to be the map from $\pskewedt$ to $\pskewedtt$ to show that $\pskewedt$ converges to $\U$.
If we assume that the goal-conditioned policy and generative model learning procedure are well behaved
(i.e., the maps from $\pgt$ to $\pet$ and from $\pskewedt$ to $\pgtt$ are continuous),
then to apply Lemma~\ref{lemma:general-convergence}, we only need to show that \mbox{$\gH(\pskewedt) \geq \gH(\pet)$} with equality if and only if \mbox{$\pet = \U$}.
For the simple case when \mbox{$\pgt = \pet$} identically at each iteration, we prove the convergence of \METHOD true for any value of $\alpha \in [-1, 0)$ in \autoref{sec:simple-case-proof}.
However, in practice, $\pgt$ only approximates $\pet$. To address this more realistic situation, we prove the following result:
\begin{lemma}\label{lemma:pos-cov-negative-grad}
Given two distribution $\pet$ and $\pgt$ where $\pet \ll \pgt$
\footnote{
$p \ll q$ means that $p$ is absolutely continuous with respect to $q$, i.e. $p(\st) = 0 \implies q(\st) = 0$.
}
and
\begin{align}\label{eq:pos-cov}
  \Cov_{\St \sim \pet}\left[\log \pet(\St), \log \pgt(\St)\right] > 0,
\end{align}
define the $\pskewedt$ as in \Eqref{eq:pskew-defn} and take $N \rightarrow \infty$.
Let $\gH_\alpha(\alpha)$ be the entropy of $\pskewedt$ for a fixed $\alpha$.
Then there exists a constant $a < 0$ such that for all $\alpha \in [a, 0)$,
\begin{align*}
    \gH(\pskewedt) =  \gH_\alpha(\alpha) > \gH(\pet).
\end{align*}
\end{lemma}
\begin{proof}
See Appendix Section \ref{sec:covariance-proof}.
\end{proof}
This lemma tells us that our generative model $\pgt$ does not need to exactly fit the sampled states.
Rather, we merely need the log densities of $\pgt$ and $\pet$ to be correlated, which we expect to happen frequently with an accurate goal-conditioned policy, since $\pet$ is the set of states seen when trying to reach goals from $\pgt$.
In this case, if we choose negative values of $\alpha$ that are small enough, then the entropy of $\pskewedt$ will be higher than that of $\pet$.
Empirically, we found that $\alpha$ values as low as $\alpha=-1$ performed well.

In summary, $\pskewedt$ converges to $\U$ under certain assumptions.
Since we train each generative model $\pgtt$ by fitting it to $\pskewedt$ with MLE, $\pgt$ will also converge to $\U$.

%% file: training_policy.tex
Thus far, we have presented \METHOD assuming that we have access to a goal-reaching policy, allowing us to separately analyze how we can maximize $\HG$.
However, in practice we do not have access to such a policy, and this section discusses how we concurrently train a goal-reaching policy.

Maximizing $I(\SF; \G)$ can be done by simultaneously performing \METHOD and training a goal-conditioned policy to minimize $\HGS$, or, equivalently, maximize $-\HGS$.
Maximizing $-\HGS$ requires computing the density $\log p(\G \mid \SF)$, which may be difficult to compute without strong modeling assumptions.
However, for any distribution $q$, the following lower bound on $-\HGS$:
\begin{align}\nonumber
-\HGS
    &= \E_{(\G, \SF) \sim q}\left[
        \log q(\G \mid \SF)
    \right]
+ \kld{p}{q}
\\\nonumber
&
    \geq \E_{(\G, \SF) \sim q}\left[
        \log q(\G \mid \SF)
    \right],
\end{align}
where $\KL$ denotes Kullback–Leibler divergence as discussed by \citet{barber2004information}.
Thus, to minimize $\HGS$, we train a policy to maximize the reward
\begin{align}\nonumber
    r(\SF, \G) = \log q(\G \mid \SF).
\end{align}

The RL algorithm we use is reinforcement learning with imagined goals (RIG)~\citep{nair2018rig}, though in principle any goal-conditioned method could be used.
RIG is an efficient off-policy goal-conditioned method that solves vision-based RL problems in a learned latent space.
In particular, RIG fits a $\beta$-VAE~\citep{higgins2016beta} and uses it to encode observations and goals into a latent space, which it uses as the state representation.
RIG also uses the $\beta$-VAE to compute rewards, $\log q(\G \mid \SF)$.
Unlike RIG, we use the goal distribution from \METHOD to sample goals for exploration and for relabeling goals during training~\citep{andrychowicz2017her}.
Since RIG already trains a generative model over states, we reuse this $\beta$-VAE for the generative model $\pg$ of \METHOD.
To make the most use of the data, we train $\pg$ on all visited state rather than only the terminal states, which we found to work well in practice.
To prevent the estimated state likelihoods from collapsing to zero, we model the posterior of the $\beta$-VAE as a multivariate Gaussian distribution with a fixed variance and only learn the mean.
We summarize RIG and provide details for how we combine \METHOD and RIG in \autoref{sec:rig-and-full-method} and describe how we estimate the likelihoods given the $\beta$-VAE in \autoref{sec:likelihood-estimation-vae}.

%% file: related_work.tex
Many prior methods in the goal-conditioned reinforcement learning literature focus on training goal-conditioned policies and assume that a goal distribution is available to sample from during exploration~\citep{kaelbling1993goals,schaul2015uva,andrychowicz2017her,pong2018tdm}, or use a heuristic to design a non-parametric~\citep{colas2018gep,wardefarley2018discern,florensa2018selfsupervised} or parametric~\citep{pere2018unsupervised,nair2018rig} goal distribution based on previously visited states.
These methods are largely complementary to our work:
rather than proposing a better method for training goal-reaching policies, we propose a principled method for maximizing the entropy of a goal sampling distribution, $\gH(\G)$, such that these policies cover a wide range of states.

Our method learns without any task rewards, directly acquiring a policy that can be reused to reach user-specified goals.
This stands in contrast to exploration methods that modify the reward based on state visitation frequency~\citep{bellemare2016unifying,ostrovski2017count,tang2017hashtag,chentanez2005intrinsically,lopes2012exploration,stadie2016exploration,pathak2017curiosity,burda2018exploration,burda2018large,mohamed2015variational,tang2017hashtag,fu2017ex2}.
While these methods can also be used without a task reward, they provide no mechanism for distilling the knowledge gained from visiting diverse states into flexible policies that can be applied to accomplish new goals at test-time: their policies visit novel states, and they quickly forget about them as other states become more novel.
Similarly, methods that provably maximize state entropy without using goal-directed exploration~\citep{hazan2019provably} or methods that define new rewards to capture measures of intrinsic motivation~\citep{mohamed2015variational} and reachability~\citep{savinov2018episodic} do not produce reusable policies.

Other prior methods extract reusable skills in the form of latent-variable-conditioned policies, where
latent variables are interpreted as options~\citep{sutton1999between} or abstract skills~\citep{hausman2018skillembedding,gupta2018structuredexploration,eysenbach2018diayn,gupta2018unsupervised,florensa2017stochastic}.
The resulting skills are diverse, but have no grounded interpretation, while \METHOD policies can be used immediately after unsupervised training to reach diverse user-specified goals.

Some prior methods propose to choose goals based on heuristics such as learning progress~\citep{baranes2012, veeriah2018many, colas2018curious}, how off-policy the goal is~\citep{nachum2018hiro}, level of difficulty~\citep{held2018goalgan}, or likelihood ranking~\citep{zhao2019rankweight}.
In contrast, our approach provides a principled framework for optimizing a concrete and well-motivated exploration objective, can provably maximize this objective under regularity assumptions, and empirically outperforms many of these prior work (see \autoref{sec:experiments}).

%% file: experiments.tex
Our experiments study the following questions:
\textbf{(1)} Does \METHOD empirically result in a goal distribution with increasing entropy?
\textbf{(2)} Does \METHOD improve exploration for goal-conditioned RL?
\textbf{(3)} How does \METHOD compare to prior work on choosing goals for vision-based, goal-conditioned RL?
\textbf{(4)} Can \METHOD be applied to a real-world, vision-based robot task?

\begin{figure}[t]
    \includegraphics[width=.49\linewidth ]{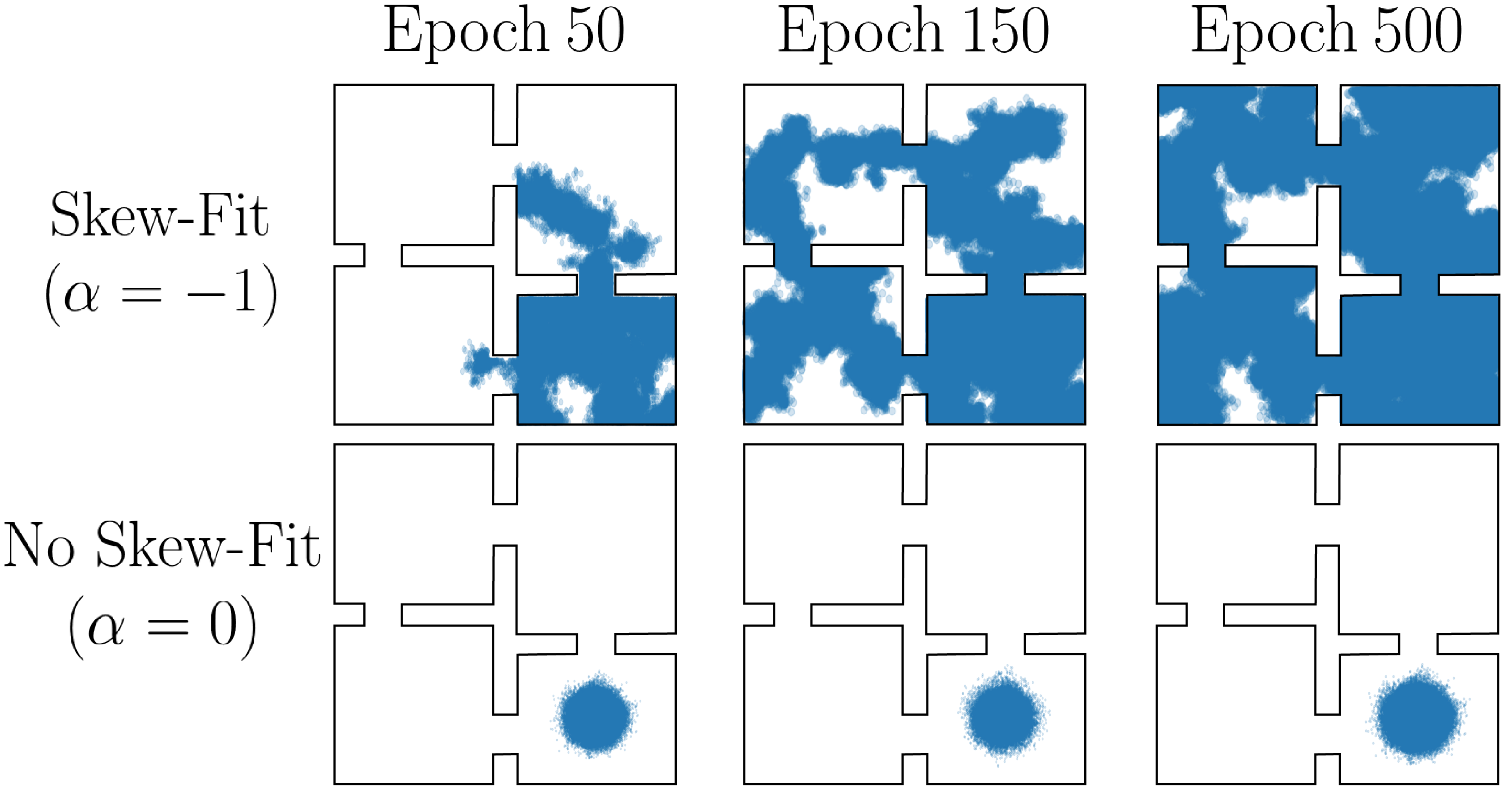}
    \includegraphics[width=.49\linewidth ]{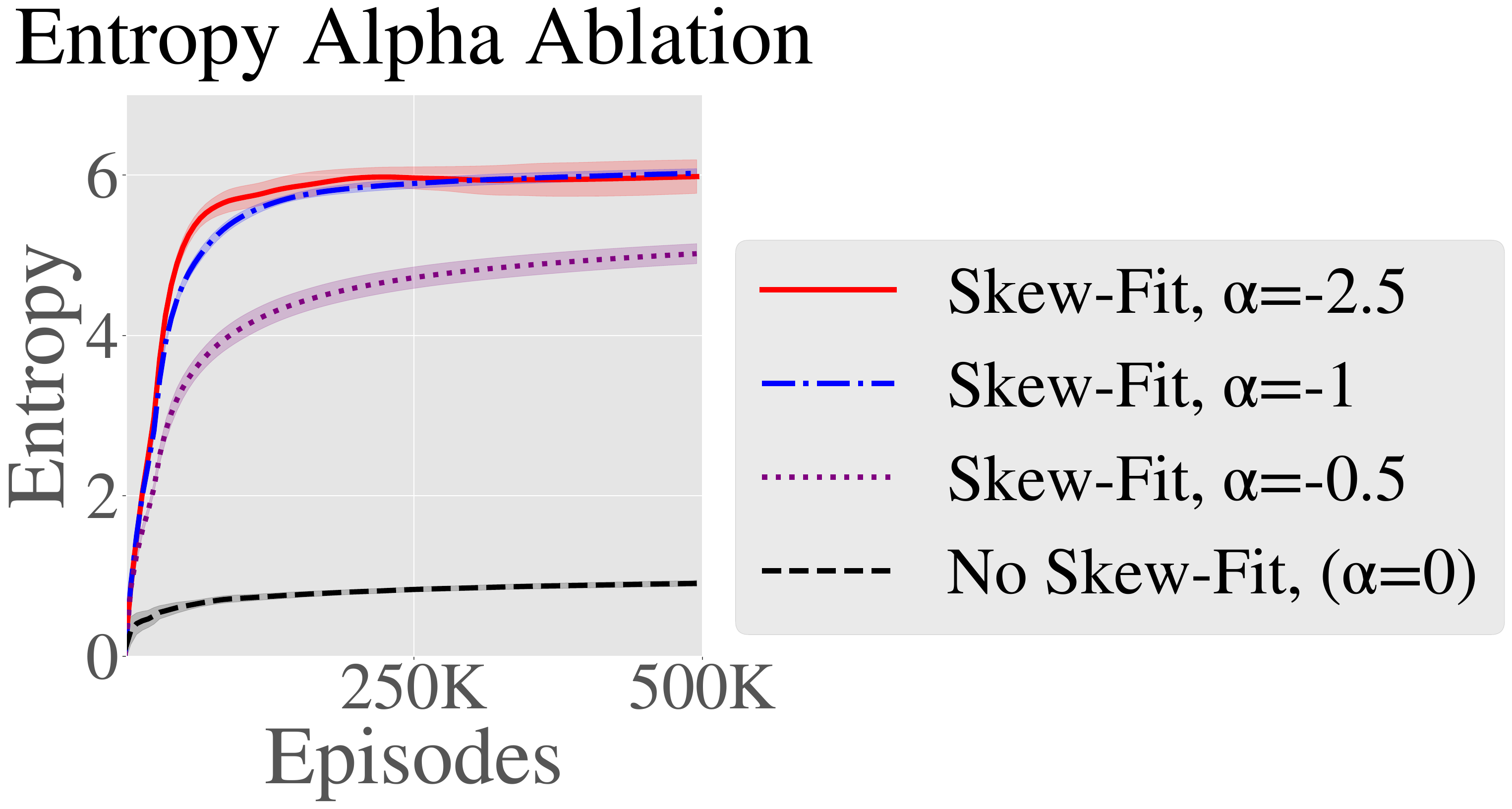}
    \fcaption{
    Illustrative example of \METHOD on a 2D navigation task. (Left) Visited state plot for \METHOD with $\alpha = -1$ and uniform sampling, which corresponds to $\alpha = 0$. (Right) The entropy of the goal distribution per iteration, mean and standard deviation for 9 seeds. Entropy is calculated via discretization onto an 11x11 grid. \METHOD steadily increases the state entropy, reaching full coverage over the state space.
    }
    \label{fig:2d-sl}
\end{figure}

\paragraph{Does \METHOD Maximize Entropy?}
To see the effects of \METHOD on goal distribution entropy in isolation of learning a goal-reaching policy, we study an idealized example where the policy is a near-perfect goal-reaching policy.
The environment consists of four rooms~\citep{sutton1999between}.
At the beginning of an episode, the agent begins in the bottom-right room and samples a goal from the goal distribution $\pgt$.
To simulate stochasticity of the policy and environment, we add a Gaussian noise with standard deviation of $0.06$ units to this goal, where the entire environment is $11 \times 11$ units.
The policy reaches the state that is closest to this noisy goal and inside the rooms, giving us a state sample $\st_n$ for training $\pgt$.
Due to the relatively small noise, the agent cannot rely on this stochasticity to explore the different rooms and must instead learn to set goals that are progressively farther and farther from the initial state.
We compare multiple values of $\alpha$, where $\alpha=0$ corresponds to not using \METHOD.
The $\beta$-VAE hyperparameters used to train $\pgt$ are given in \autoref{sec:2d-details}.
As seen in \Figref{fig:2d-sl}, sampling uniformly from previous experience ($\alpha = 0$) to set goals results in a policy that primarily sets goal near the initial state distribution.
In contrast, \METHOD results in quickly learning a high entropy, near-uniform distribution over the state space.

\begin{figure}[t]
    \centering
    \includegraphics[width=0.32\linewidth]{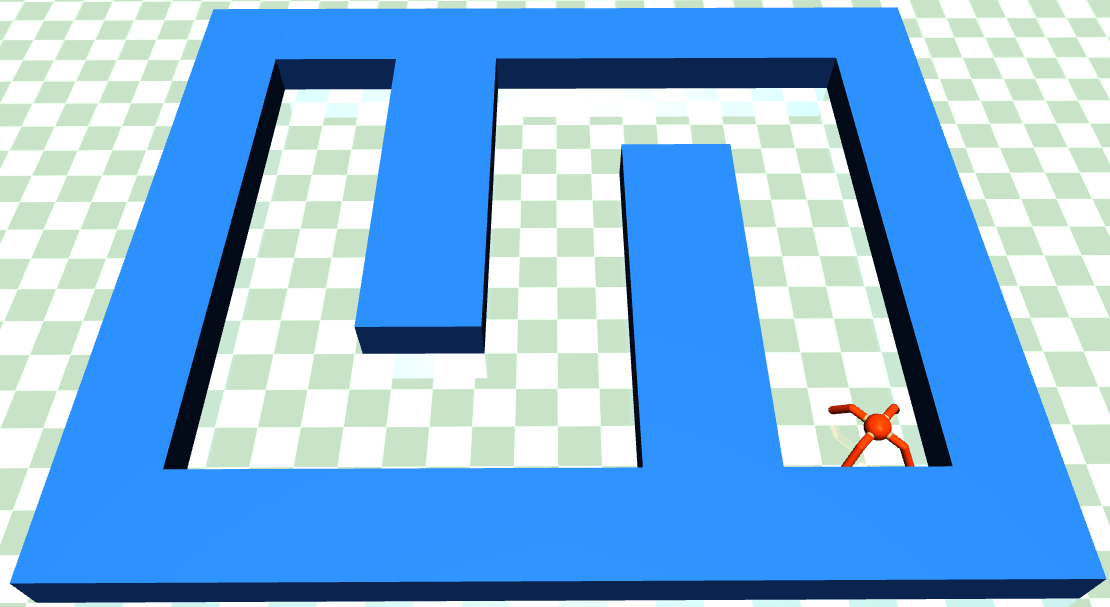}
    \includegraphics[width=0.65\linewidth]{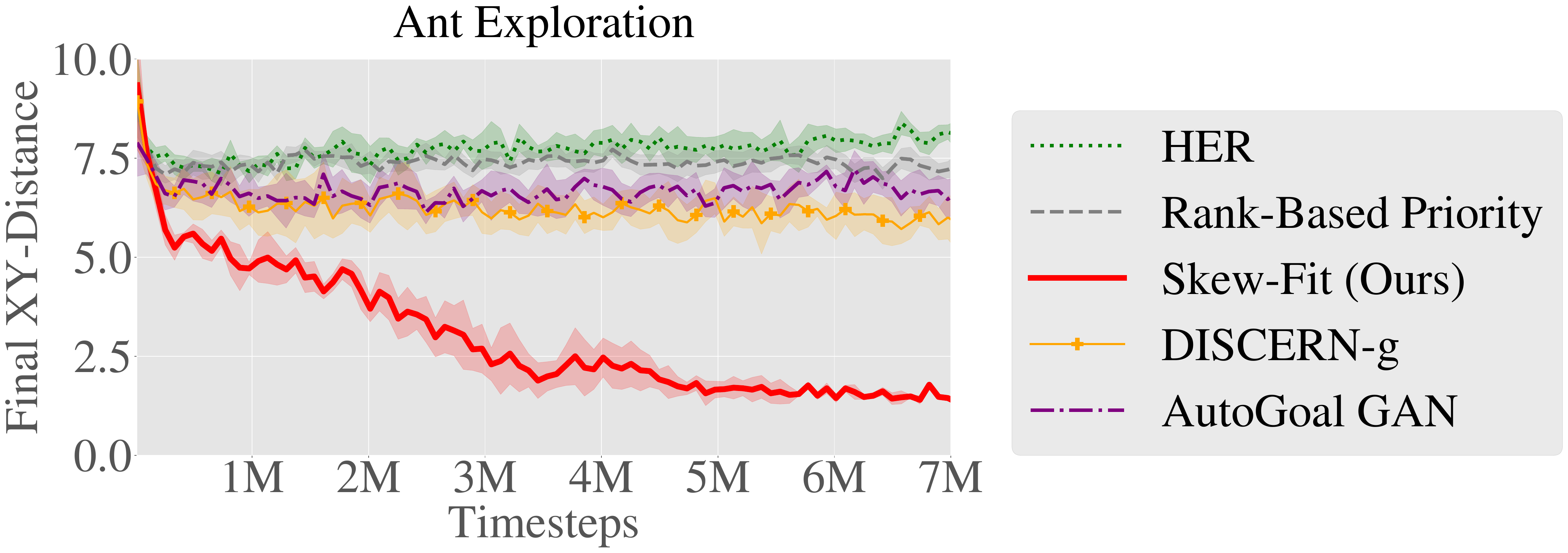}
    \fcaption{
    (Left) Ant navigation environment.
    (Right) Evaluation on reaching target XY position.
    We show the mean and standard deviation of 6 seeds.
    \METHOD significantly outperforms prior methods on this exploration task.
    }
    \label{fig:antmaze}
\end{figure}

\paragraph{Exploration with \METHOD}
We next evaluate \METHOD while concurrently learning a goal-conditioned policy on a task with state inputs, which enables us study exploration performance independently of the challenges with image observations.
We evaluate on a task that requires training a simulated quadruped ``ant'' robot to navigate to different XY positions in a labyrinth,
as shown in \Figref{fig:antmaze}.
The reward is the negative distance to the goal XY-position, and additional environment details are provided in \autoref{sec:environment-details}.
This task presents a challenge for goal-directed exploration:
the set of valid goals is unknown due to the walls, and
random actions do not result in exploring locations far from the start.
Thus, \METHOD must set goals that meaningfully explore the space while simultaneously learning to reach those goals.

We use this domain to compare \METHOD to a number of existing goal-sampling methods.
We compare to the relabeling scheme described in the hindsight experience replay (labeled \textbf{HER}).
We compare to curiosity-driven prioritization (\textbf{Ranked-Based Priority})~\citep{zhao2019maximum}, a variant of HER that samples goals for relabeling based on their ranked likelihoods.
\citet{held2018goalgan} samples goals from a GAN based on the difficulty of reaching the goal.
We compare against this method by replacing $\pg$ with the GAN and label it \textbf{AutoGoal GAN}.
We also compare to the non-parametric goal proposal mechanism proposed by \cite{wardefarley2018discern}, which we label \textbf{DISCERN-g}.
Lastly, to demonstrate the difficulty of the exploration challenge in these domains, we compare to \textbf{\#-Exploration}~\citep{tang2017hashtag}, an exploration method that assigns bonus rewards based on the novelty of new states.
We train the goal-conditioned policy for each method using soft actor critic (SAC)~\citep{haarnoja2018sacapp}.
Implementation details of SAC and the prior works are given in  \autoref{sec:prior-work-implementation}.

We see in \Figref{fig:antmaze} that \METHOD is the only method that makes significant progress on this challenging labyrinth locomotion task.
The prior methods on goal-sampling primarily set goals close to the start location, while the extrinsic exploration reward in \#-Exploration dominated the goal-reaching reward.
These results demonstrate that \METHOD accelerates exploration by setting diverse goals in tasks with unknown goal spaces.

\paragraph{Vision-Based Continuous Control Tasks}
\begin{figure}[t]
    \centering
    \includegraphics[width=0.8\linewidth]{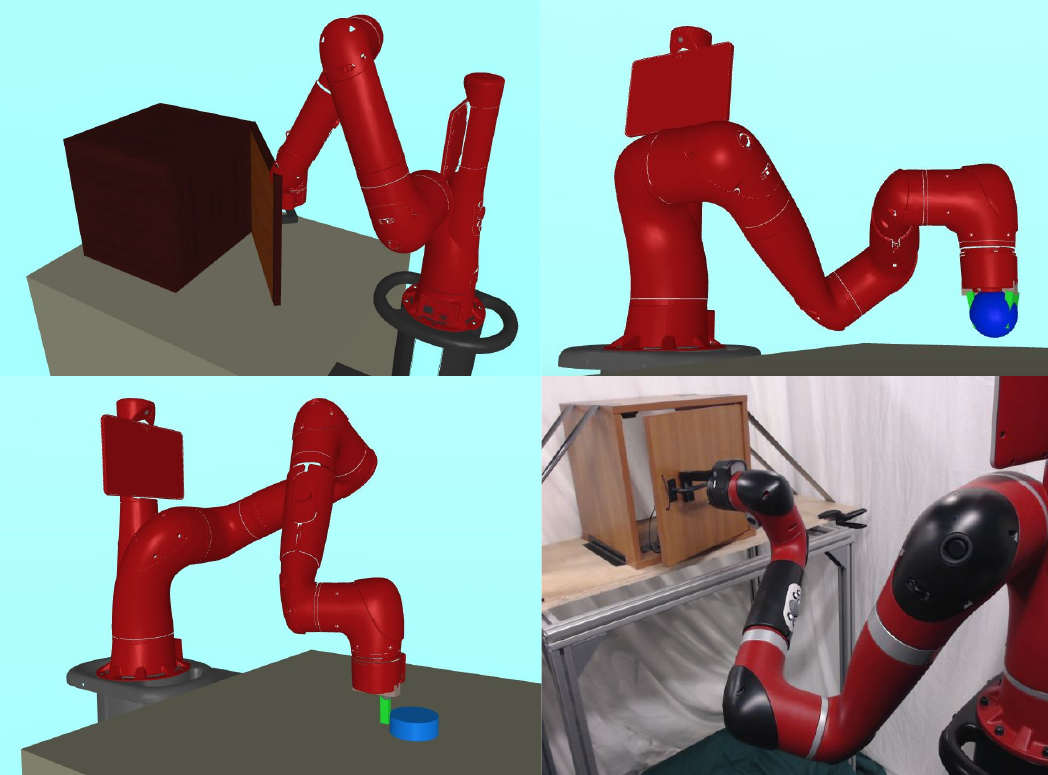}
    \fcaption{We evaluate on these continuous control tasks, from left to right:
    \textit{Visual Door}, a door opening task;
    \textit{Visual Pickup}, a picking task;
    \textit{Visual Pusher}, a pushing task;
    and \textit{Real World Visual Door}, a real world door opening task. All tasks are solved from images and without any task-specific reward. See Appendix \ref{sec:environment-details} for details.}
    \label{fig:env-pics}
\end{figure}

We now evaluate \METHOD on a variety of image-based continuous control tasks, where the policy must control a robot arm using only image observations, there is no state-based or task-specific reward, and \METHOD must directly set image goals.
We test our method on three different image-based simulated continuous control tasks released by the authors of RIG~\citep{nair2018rig}: \textit{Visual Door}, \textit{Visual Pusher}, and \textit{Visual Pickup}.
These environments contain a robot that can open a door, push a puck, and lift up a ball to different configurations, respectively.
To our knowledge, these are the only goal-conditioned, vision-based continuous control environments that are publicly available and experimentally evaluated in prior work, making them a good point of comparison.
See \autoref{fig:env-pics} for visuals and \autoref{sec:implementation-details} for environment details.
The policies are trained in a completely unsupervised manner, without access to any prior information about the image-space or any pre-defined goal-sampling distribution.
To evaluate their performance, we sample goal images from a uniform distribution over valid states and report the agent's final distance to the corresponding simulator states (e.g., distance of the object to the target object location), but the agent never has access to this true uniform distribution nor the ground-truth state information during training.
While this evaluation method is only practical in simulation, it provides us with a quantitative measure of a policy's ability to reach a broad coverage of goals in a vision-based setting.

We compare \METHOD to a number of existing methods on this domain.
First, we compare to the methods described in the previous experiment (HER, Rank-Based Priority, \#-Exploration, Autogoal GAN, and \mbox{DISCERN-g}).
These methods that we compare to were developed in non-vision, state-based environments.
To ensure a fair comparison across methods, we combine these prior methods with a policy trained using RIG.
We additionally compare to \citet{hazan2019provably}, an exploration method that assigns bonus rewards based on the likelihood of a state (labeled \textbf{Hazan et al.}).
Next, we compare to \textbf{RIG} without \METHOD.
Lastly, we compare to \textbf{DISCERN}~\citep{wardefarley2018discern}, a vision-based method which uses a non-parametric clustering approach to sample goals and an image discriminator to compute rewards.

\begin{figure}
    \centering
     \begin{subfigure}[t]{.49\linewidth}
    \centering
          \includegraphics[width=\linewidth]{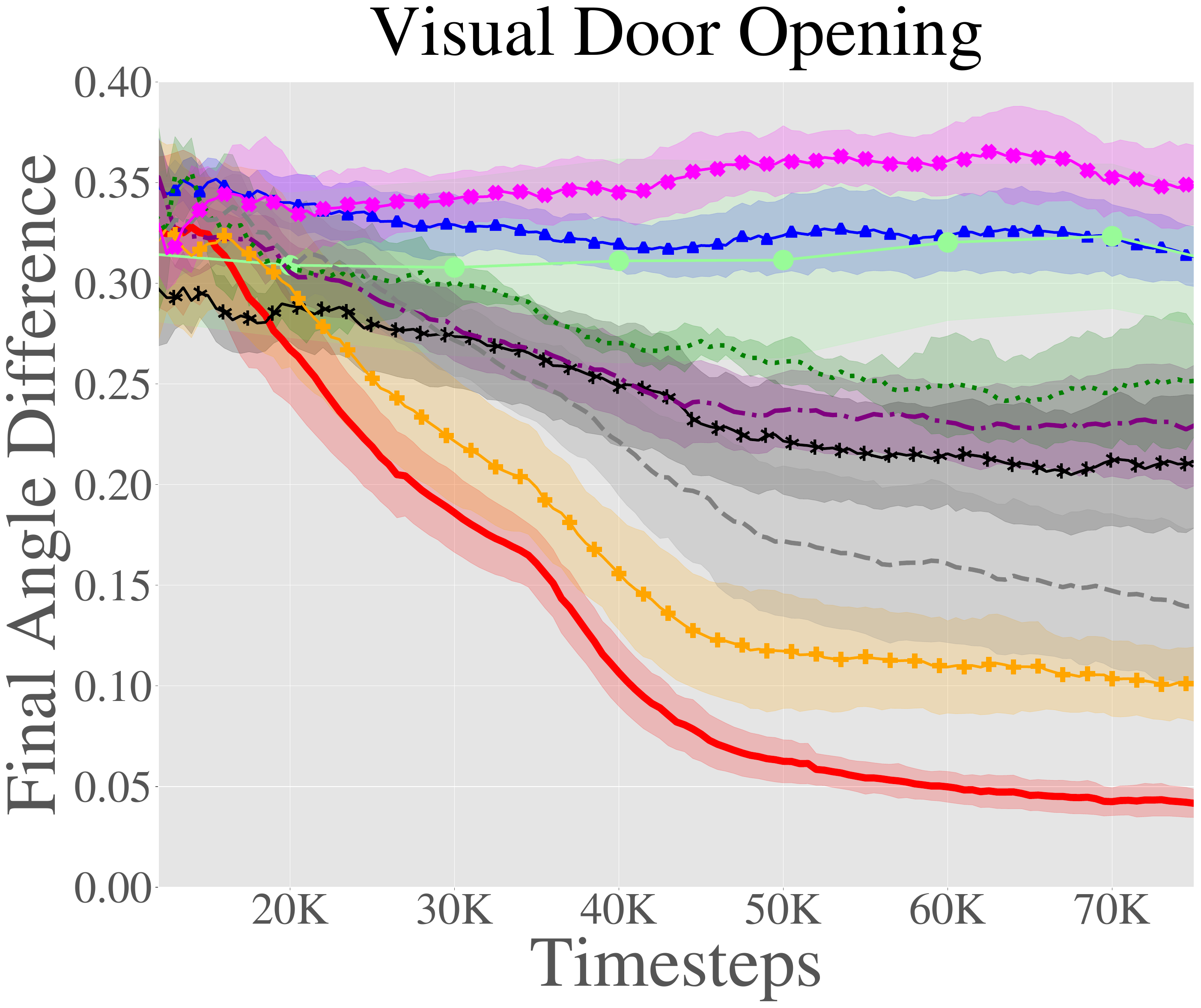}
  \end{subfigure}
  \hfill
  \begin{subfigure}[t]{.49\linewidth}
    \centering
          \includegraphics[width=\linewidth]{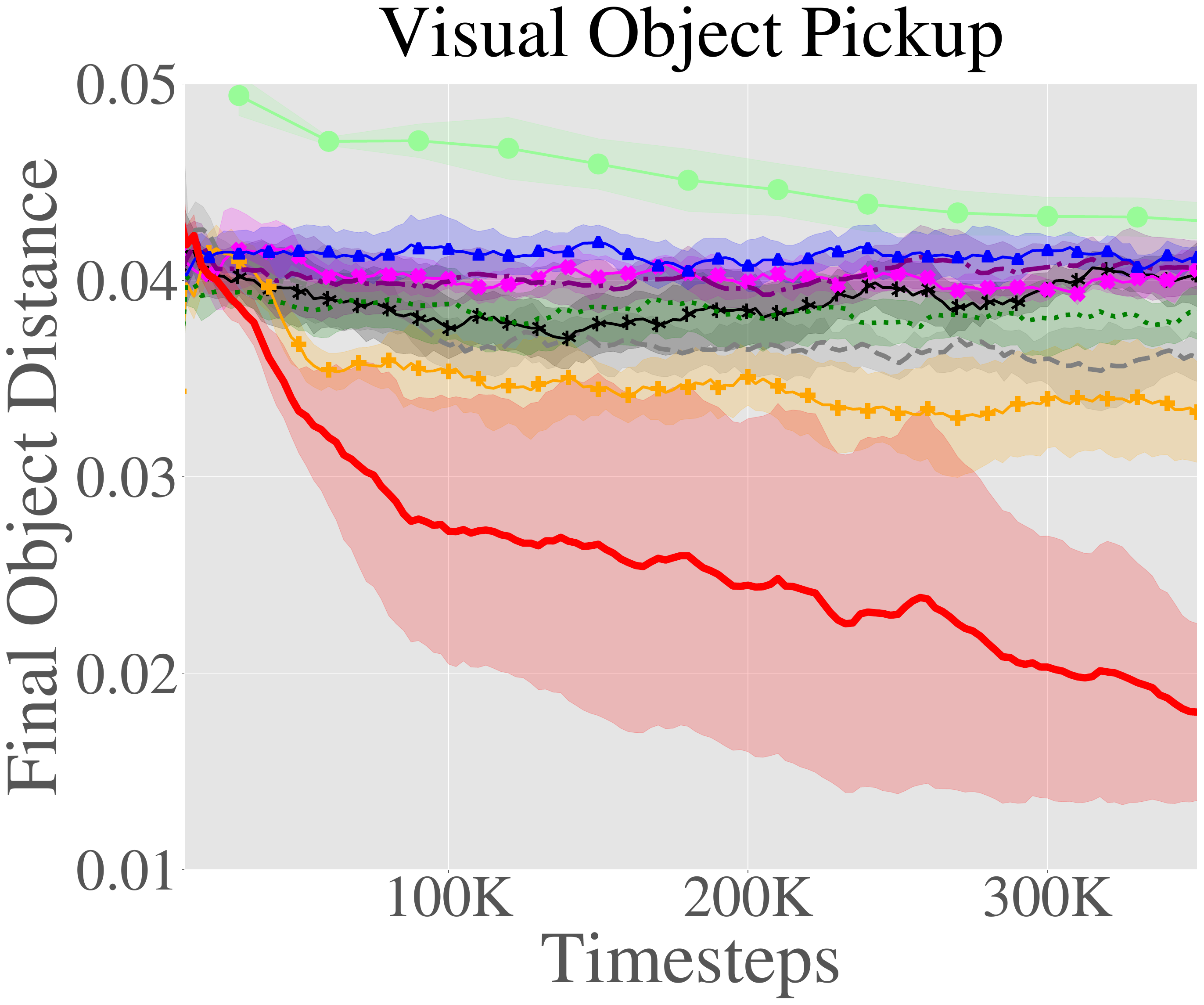}
  \end{subfigure}

  \medskip

  \begin{subfigure}[t]{.49\linewidth}
    \centering
          \includegraphics[width=\linewidth]{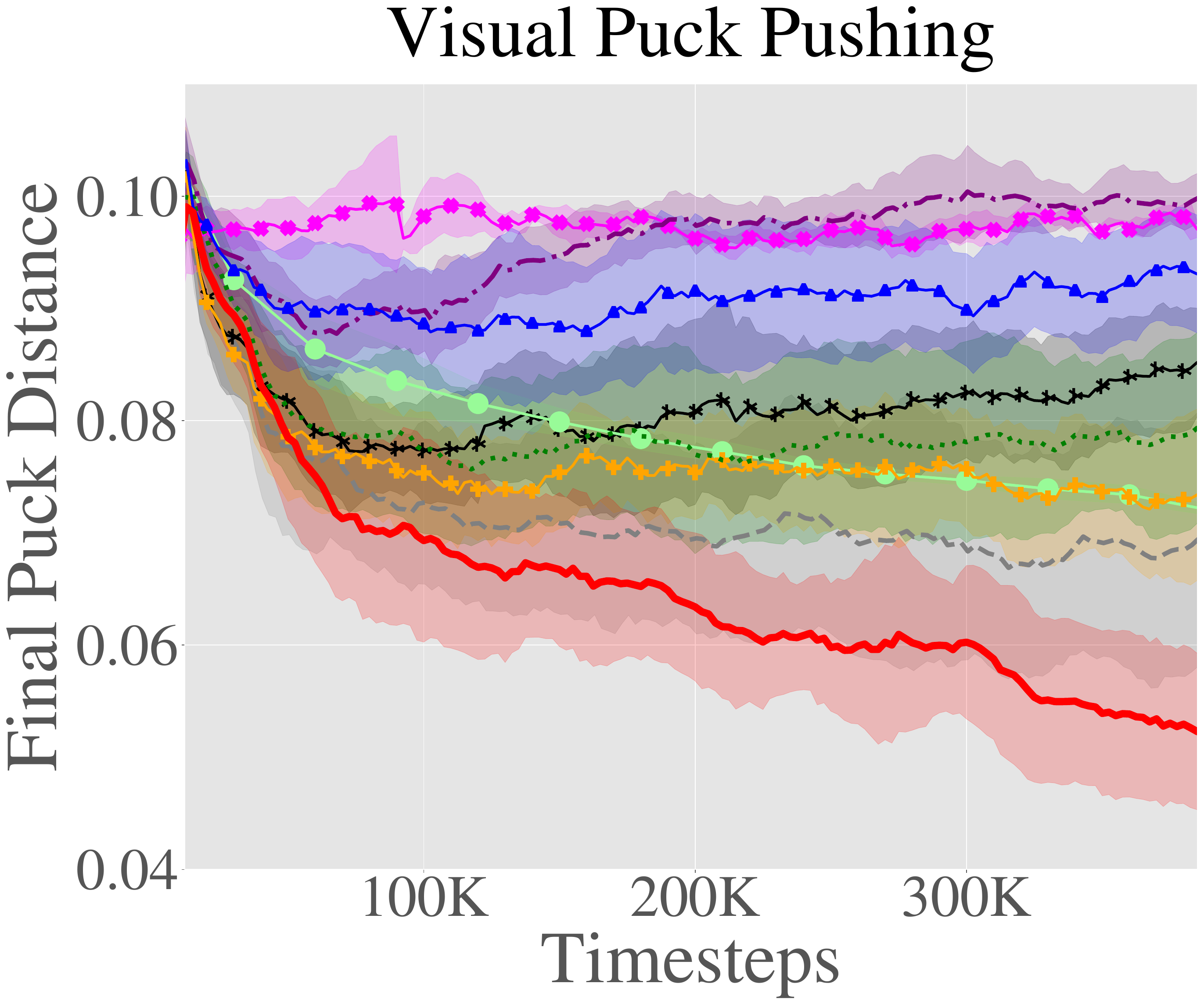}
  \end{subfigure}
  \hfill
  \begin{subfigure}[t]{.48\linewidth}
    \centering
    \raisebox{0.16in}{
          \includegraphics[width=\linewidth]{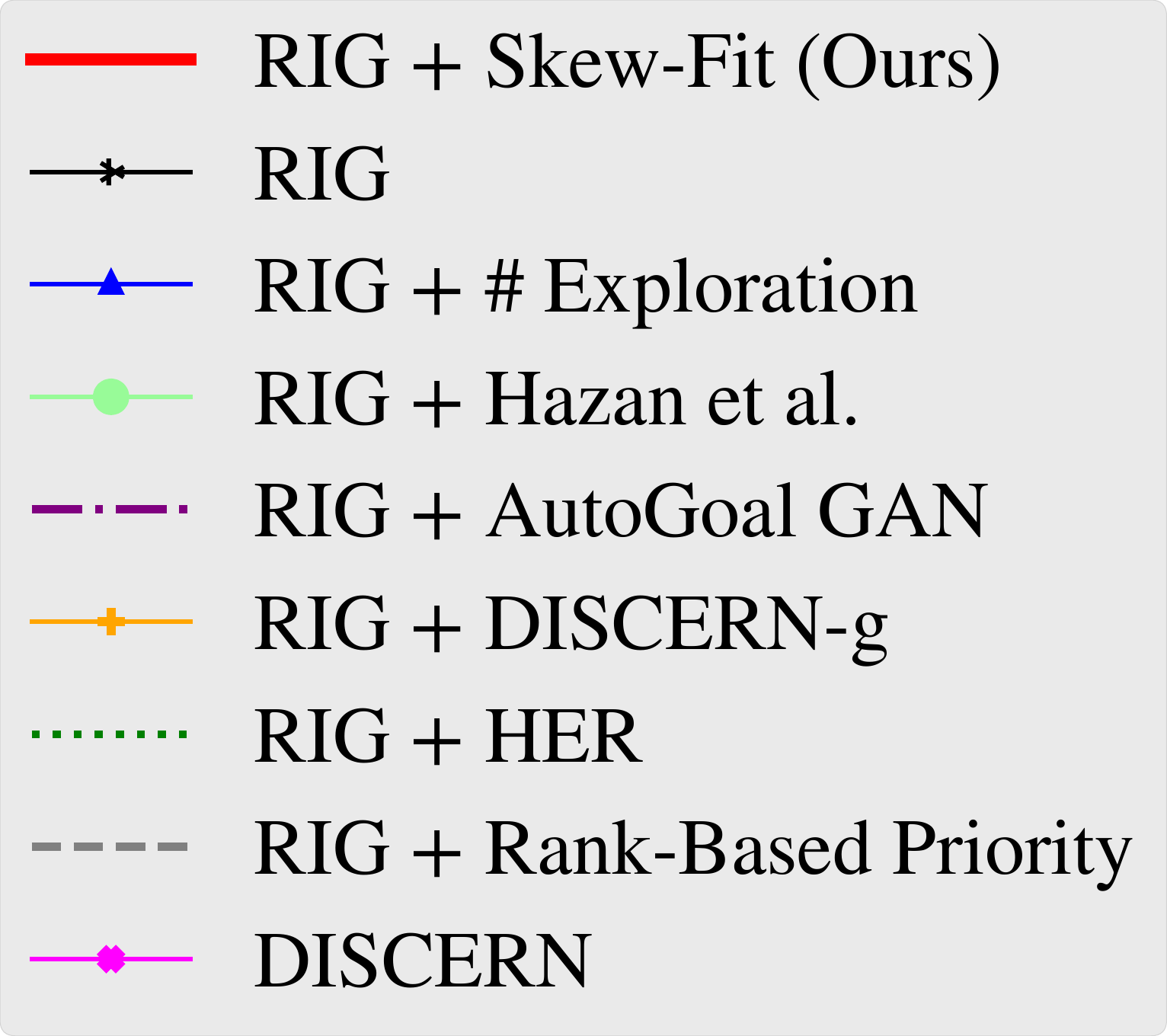}
    }
  \end{subfigure}
    \fcaption{
        Learning curves for simulated continuous control tasks.
        Lower is better.
        We show the mean and standard deviation of 6 seeds and smooth temporally across 50 epochs within each seed.
        \METHOD consistently outperforms RIG and various prior methods.
        See text for description of each method.
    }
    \label{fig:sim-results}
\end{figure}

We see in \Figref{fig:sim-results} that Skew-Fit significantly outperforms prior methods both in terms of task performance and sample complexity.
The most common failure mode for prior methods is that the goal distributions collapse, resulting in the agent learning to reach only a fraction of the state space, as shown in \autoref{fig:offline-sk-real}.
For comparison, additional samples of $\pg$ when trained with and without \METHOD are shown in \autoref{sec:vae-dump}.
Those images show that without \METHOD, $\pg$ produces a small, non-diverse distribution for each environment: the object is in the same place for pickup, the puck is often in the starting position for pushing, and the door is always closed.
In contrast, \METHOD proposes goals where the object is in the air and on the ground, where the puck positions are varied, and the door angle changes.

We can see the effect of these goal choices by visualizing more example rollouts for RIG and \METHOD.
These visuals, shown in \Figref{fig:example_rollouts} in \autoref{sec:vae-dump}, show that RIG only learns to reach states close to the initial position, while \METHOD learns to reach the entire state space.
For a quantitative comparison, \Figref{fig:exploration_pickups} shows the cumulative total exploration pickups for each method.
From the graph, we see that many methods have a near-constant rate of object lifts throughout all of training.
\METHOD is the only method that significantly increases the rate at which the policy picks up the object during exploration, suggesting that only \METHOD sets goals that encourage the policy to interact with the object.

\begin{figure}[t]
\centering
  \includegraphics[width=\linewidth]{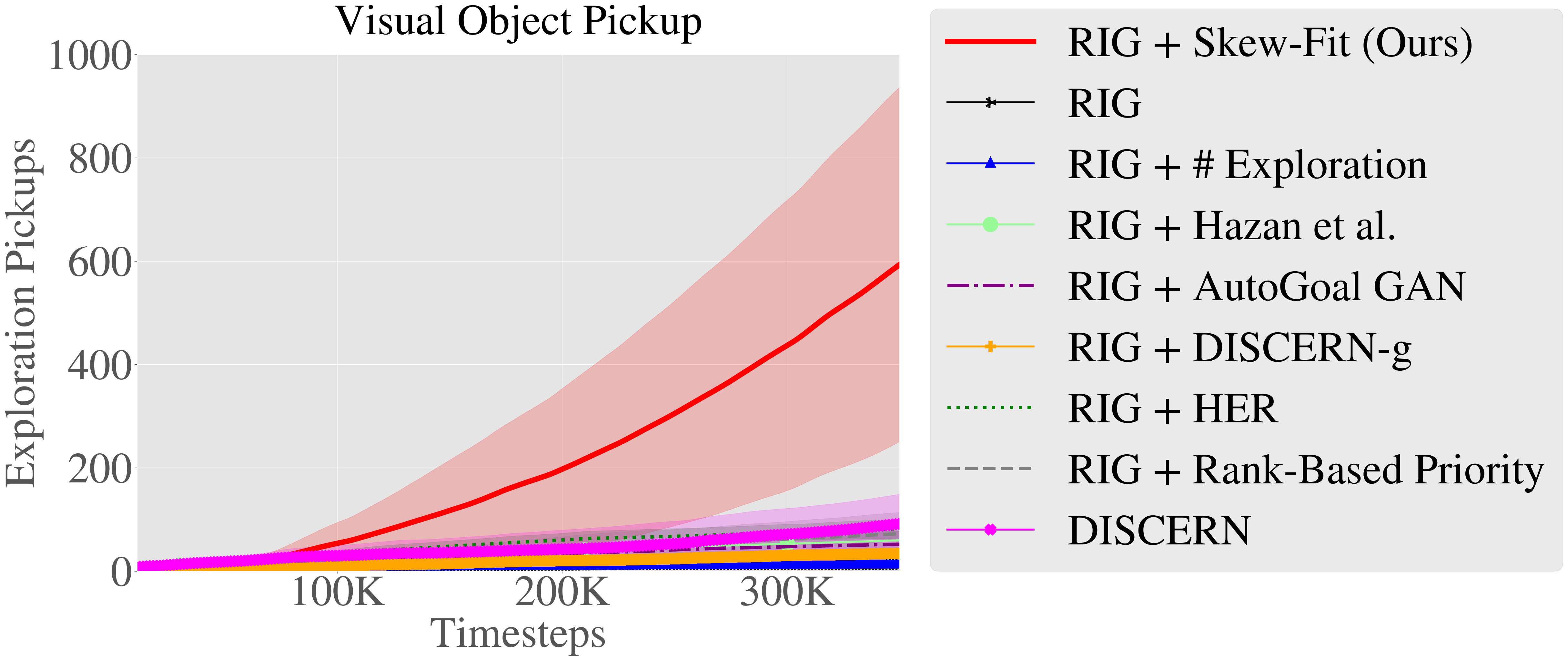}
  \fcaption{
Cumulative total pickups during exploration for each method.
Prior methods fail to pay attention to the object: the rate of pickups hardly increases past the first 100 thousand timesteps.
In contrast, after seeing the object picked up a few times, \METHOD practices picking up the object more often by sampling the appropriate exploration goals.
}
  \label{fig:exploration_pickups}
\end{figure}

\paragraph{Real-World Vision-Based Robotic Manipulation}
We also demonstrate that \METHOD scales well to the real world with a door opening task, \textit{Real World Visual Door}, as shown in \Figref{fig:env-pics}.
While a number of prior works have studied RL-based learning of door opening~\cite{kalakrishnan2011learning,chebotar2017path}, we demonstrate the first method for autonomous learning of door opening without a user-provided, task-specific reward function.
As in simulation, we do not provide any goals to the agent and simply let it interact with the door, without any human guidance or reward signal.
We train two agents using RIG and RIG with \METHOD.
Every seven and a half minutes of interaction time, we evaluate on $5$ goals and plot the cumulative successes for each method.
Unlike in simulation, we cannot easily measure the difference between the policy's achieved and desired door angle.
Instead, we visually denote a binary success/failure for each goal based on whether the last state in the trajectory achieves the target angle.
As \Figref{fig:real-results} shows, standard RIG only starts to open the door after five hours of training.
In contrast, \METHOD learns to occasionally open the door after three hours of training and achieves a near-perfect success rate after five and a half hours of interaction.
\autoref{fig:real-results} also shows examples of successful trajectories from the \METHOD policy, where we see that the policy can reach a variety of user-specified goals.
These results demonstrate that \METHOD is a promising technique for solving real world tasks without any human-provided reward function.
Videos of \METHOD solving this task and the simulated tasks can be viewed on
our website.
\footnote{https://sites.google.com/view/skew-fit}

\begin{figure}[t]
  \centering
  \includegraphics[width=0.75\linewidth]{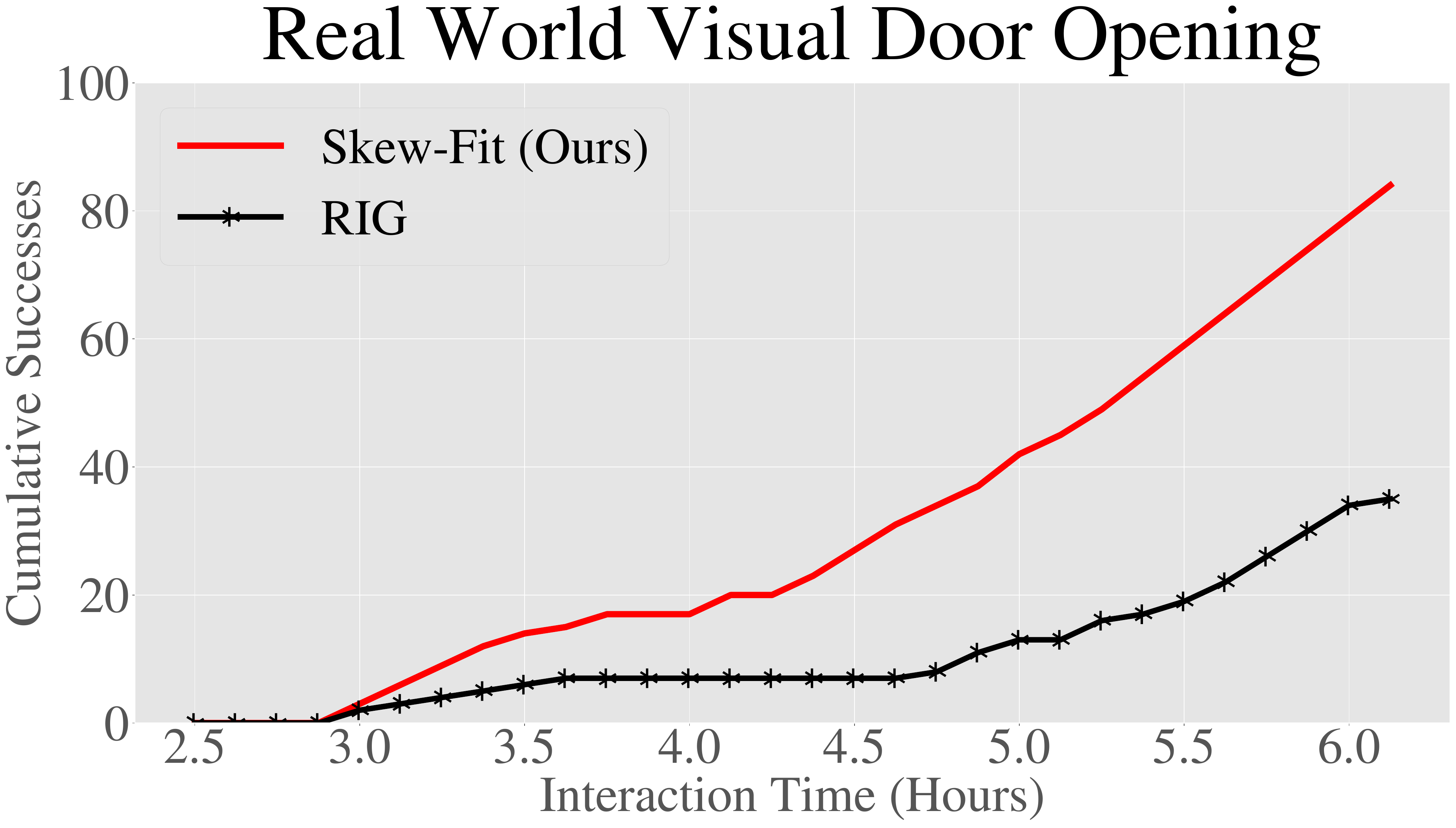}
    \begin{subfigure}[b]{0.49\textwidth}
        \center
        \hspace{-.2cm}
        $\SF_1$ \hspace{4.3cm} $\SF_{100}$ \hspace{.7cm} $\G$

        \includegraphics[width=0.14\linewidth]{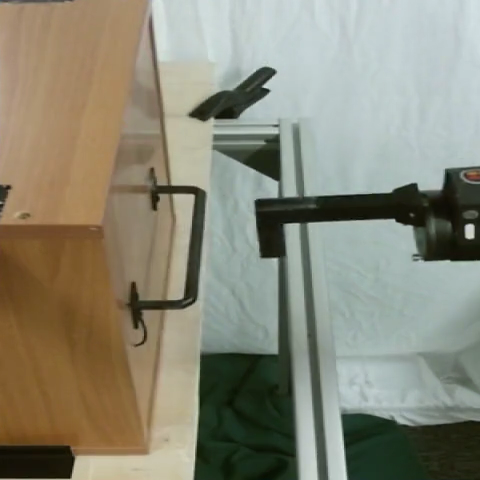}
        \includegraphics[width=0.14\linewidth]{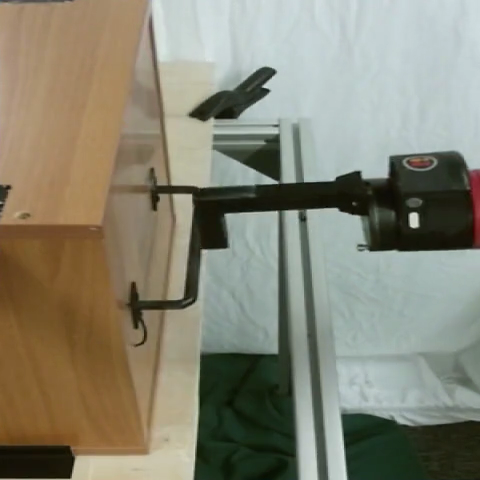}
        \includegraphics[width=0.14\linewidth]{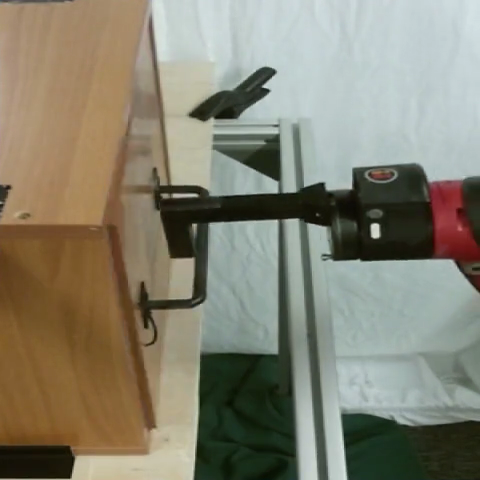}
        \includegraphics[width=0.14\linewidth]{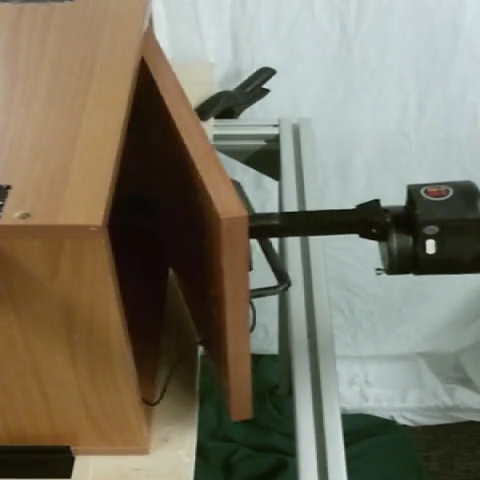}
        \includegraphics[width=0.14\linewidth]{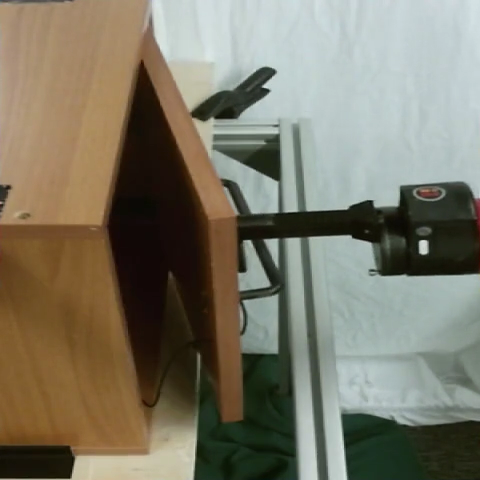}
        \hspace{0.01\linewidth}
        \includegraphics[width=0.14\linewidth]{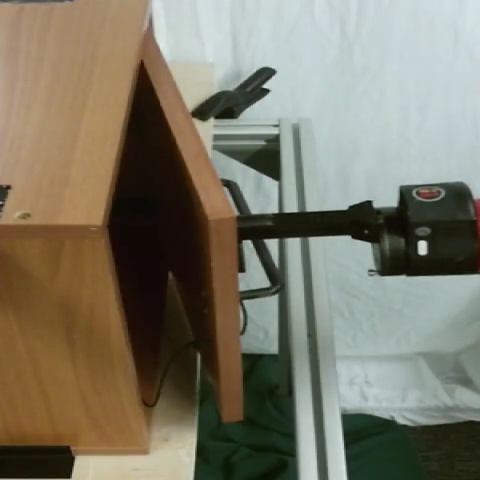} 
    \end{subfigure}
    \begin{subfigure}[b]{0.49\textwidth}
        \center

        \includegraphics[width=0.14\linewidth]{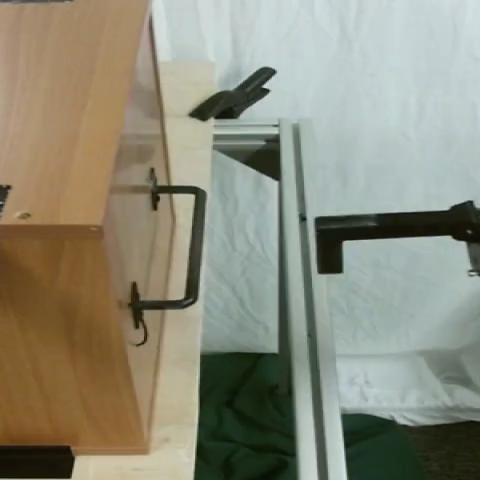}
        \includegraphics[width=0.14\linewidth]{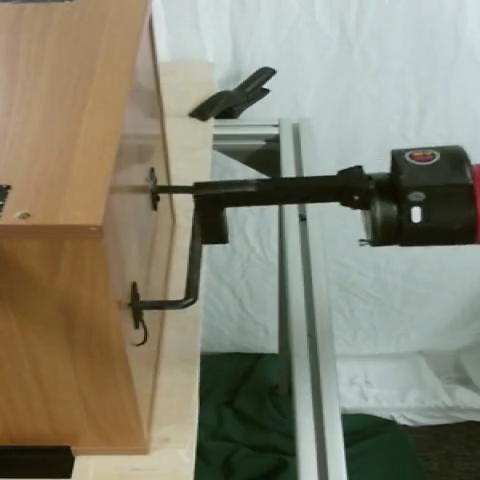}
        \includegraphics[width=0.14\linewidth]{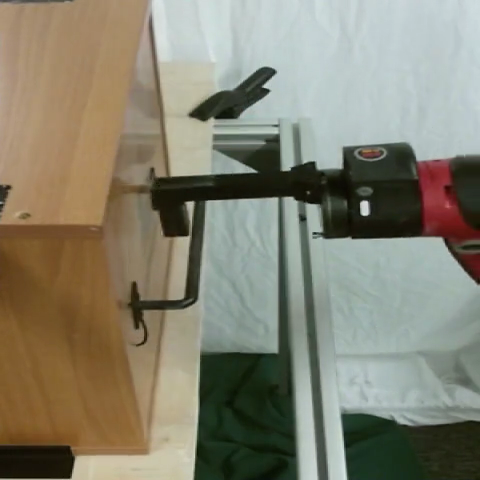}
        \includegraphics[width=0.14\linewidth]{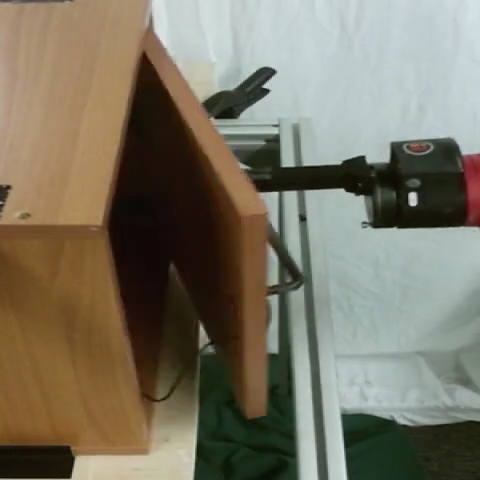}
        \includegraphics[width=0.14\linewidth]{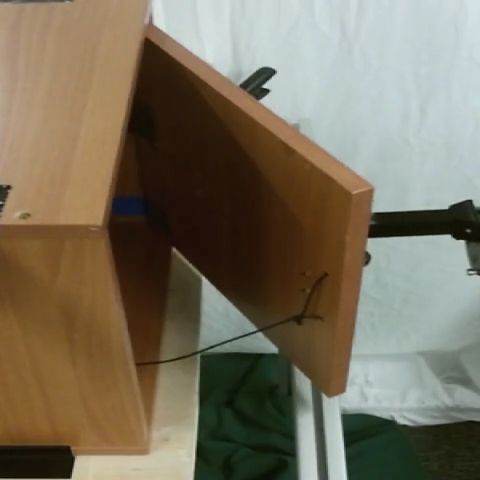}
        \hspace{0.01\linewidth}
        \includegraphics[width=0.14\linewidth]{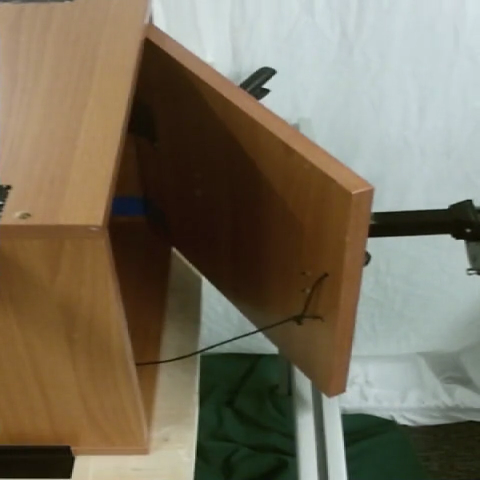}
    \end{subfigure}
  \fcaption{
  (Top) Learning curve for Real World Visual Door.
  \METHOD results in considerable sample efficiency gains over RIG on this real-world task.
  (Bottom)
  Each row shows the \METHOD policy starting from state $\SF_1$ and reaching state $\SF_{100}$ while pursuing goal $\G$.
  Despite being trained from only images without any user-provided goals during training, the \METHOD policy achieves the goal image provided at test-time, successfully opening the door.
  }
  \label{fig:real-results}
\end{figure}

\paragraph{Additional Experiments}
To study the sensitivity of \METHOD to the hyperparameter $\alpha$, we sweep $\alpha$ across the values $[-1, -0.75, -0.5, -0.25, 0]$ on the simulated image-based tasks.
The results are in \autoref{sec:add-exps} and demonstrate that \METHOD works across a large range of values for $\alpha$, and $\alpha=-1$ consistently outperform $\alpha=0$ (i.e. outperforms no \METHOD).
Additionally, \autoref{sec:implementation-details} provides a complete description our method hyperparameters, including network architecture and RL algorithm hyperparameters.

%% file: conclusion.tex
We presented a formal objective for self-supervised goal-directed exploration, allowing researchers to quantify and compare progress when designing algorithms that enable agents to autonomously learn.
We also presented \METHOD, an algorithm for training a generative model to approximate a uniform distribution over an initially unknown set of valid states, using data obtained via goal-conditioned reinforcement learning, and our theoretical analysis gives conditions under which \METHOD converges to the uniform distribution.
When such a model is used to choose goals for exploration and to relabeling goals for training, the resulting method results in much better coverage of the state space, enabling our method to explore effectively.
Our experiments show that when we concurrently train a goal-reaching policy using self-generated goals, \METHOD produces quantifiable improvements on simulated robotic manipulation tasks, and can be used to learn a door opening skill to reach a $95\%$ success rate directly on a real-world robot, without any human-provided reward supervision.

%% file: appendix.tex
\pagebreak

\appendix

\section{Proofs}\label{sec:proofs}
The definitions of continuity and convergence for pseudo-metrics are similar to those for metrics, and we state them below.

A function $f: \gQ \mapsto \gQ$ is continuous with respect to a pseudo-metric $d$ if for any $p \in \gQ$ and any scalar $\epsilon > 0$, there exists a $\delta$ such that for all \mbox{$q \in \gQ$},
\begin{align*}
    d(p, q) < \delta
    \implies
    d(f(p), f(q)) < \epsilon.
\end{align*}

An infinite sequence $p_1, p_2 \dots$ converges to a value $p$ with respect to a pseudo-metric $d$, which we write as
\begin{align*}
    \limt p_t \rightarrow p,
\end{align*}
if
\begin{align*}
    \limt d(p_t, p) \rightarrow 0.
\end{align*}

Note that if $f$ is continuous, then
\begin{align*}
    \lim_{t \rightarrow \infty} \dent(p_t, q) \rightarrow 0
    \implies
    \lim_{t \rightarrow \infty} \dent(f(p_t), f(q)) \rightarrow 0.
\end{align*}

\subsection{Proof of Lemma 3.1}\label{sec:general-proof}
\begin{lemma}
Let $\Imgs$ be a compact set.
Define the set of distributions $\gQ = \{p : \support(p) \subseteq \Imgs\}$.
Let $\gF: \gQ \mapsto \gQ$ be continuous with respect to the pseudometric \mbox{$\dent(p, q) \triangleq |\gH(p) - \gH(q)|$} and $\gH(\gF(p)) \geq \gH(p)$ with equality if and only if $p$ is the uniform probability distribution on $\Imgs$, denoted as $\U$.
Define the sequence of distributions $P = (p_1, p_2, \dots)$ by starting with any $p_1 \in \gQ$ and recursively defining $p_{t+1} = \gF(p_t)$.
The sequence $P$ converges to $\U$ with respect to $\dent$. In other words, \mbox{$\lim_{t \rightarrow 0} |\gH(p_t) - \gH(\U)| \rightarrow 0$}.
\end{lemma}
\begin{proof}
The idea of the proof is to show that the distance (with respect to $\dent$) between $p_t$ and $\U$ converges to a value.
If this value is $0$, then the proof is complete since $\U$ uniquely has zero distance to itself.
Otherwise, we will show that this implies that $\gF$ is not continuous, which a contradiction.

For shorthand, define $d_t$ to be the $\dent$-distance to the uniform distribution, as in
\begin{align*}
    d_t \triangleq \dent(p_t, \U).
\end{align*}
First we prove that $d_t$ converges.
Since the entropies of the sequence $(p_1, \dots)$ monotonically increase, we have that
\begin{align*}
    d_1 \geq d_2 \geq \dots.
\end{align*}
We also know that $d_t$ is lower bounded by $0$, and so by the monotonic convergence theorem, we have that
\begin{align*}
    \lim_{t\rightarrow\infty} d_t \rightarrow d^*.
\end{align*}
for some value $d^* \geq 0$.

To prove the lemma, we want to show that $d^* = 0$.
Suppose, for contradiction, that $d^* \neq 0$.
Then consider any distribution, $q^*$, such that $\dent(q^*, \U) = d^*$.
Such a distribution always exists since we can continuously interpolate entropy values between $\gH(p_1)$ and $\gH(\U)$ with a mixture distribution.
Note that $q^* \neq \U$ since \mbox{$\dent(\U, \U) = 0$}.
Since \mbox{$\limt d_t \rightarrow d^*$}, we have that
\begin{align}\label{eq:dent-distance-goes-to-zero}
    \limt
    \dent(p_t, q^*) \rightarrow 0,
\end{align}
and so
\begin{align*}
    \limt p_t \rightarrow q^*.
\end{align*}
Because the function $\gF$ is continuous with respect to $\dent$, \autoref{eq:dent-distance-goes-to-zero} implies that
\begin{align*}
    \limt
    \dent(\gF(p_t), \gF(q^*)) \rightarrow 0.
\end{align*}
However, since $\gF(p_t) = p_{t+1}$ we can equivalently write the above equation as
\begin{align*}
    \limt
    \dent(p_{t+1}, \gF(q^*)) \rightarrow 0.
\end{align*}
which, through a change of index variables, implies that
\begin{align*}
    \limt
    p_t \rightarrow \gF(q^*)
\end{align*}
Since $q^*$ is not the uniform distribution, we have that \mbox{$\gH(\gF(q^*)) > \gH(q^*)$}, which implies that $\gF(q^*)$ and $q^*$ are unique distributions.
So, $p_t$ converges to two distinct values, $q^*$ and $\gF(q^*)$, which is a contradiction.
Thus, it must be the case that $d^* = 0$, completing the proof.
\end{proof}

\subsection{Proof of Lemma 3.2}\label{sec:covariance-proof}
\begin{lemma}\label{lemma:covariance-general}
Given two distribution $p(x)$ and $q(x)$ where $p \ll q$ and
\begin{align}
   0 < \Cov_p[\log p(X), \log q(X)]
\end{align}
define the distribution $p_\alpha$ as
\begin{align*}
    p_\alpha(x) &= \frac{1}{Z_\alpha} p(x) q(x)^\alpha
\end{align*}
where $\alpha \in \R$ and $Z_\alpha$ is the normalizing factor.
Let $\gH_\alpha(\alpha)$ be the entropy of $p_\alpha$.
Then there exists a constant $a > 0$ such that for all $\alpha \in [-a, 0)$,
\begin{align}
    \gH_\alpha(\alpha) > \gH_\alpha(0) = \gH(p).
\end{align}

\end{lemma}
\begin{proof}
Observe that $\{p_\alpha : \alpha \in [-1, 0]\}$ is a one-dimensional exponential family
\begin{align*}
    p_\alpha(x) = e^{\alpha T(x) - A(\alpha) + k(x)}
\end{align*}
with log carrier density $k(x) = \log p(x)$, natural parameter $\alpha$, sufficient statistic $T(x) = \log q(x)$, and log-normalizer $A(\alpha) = \int_{\gX} e^{\alpha T(x) + k(x)}dx$.
As shown in \cite{nielsen2010entropies}, the entropy of a distribution from a one-dimensional exponential family with parameter $\alpha$ is given by:
\begin{align*}
    \gH_\alpha(\alpha) \triangleq
    \gH(p_\alpha)
    = A(\alpha) - \alpha A'(\alpha) - \E_{p_\alpha}[k(X)]
\end{align*}
The derivative with respect to $\alpha$ is then
\begin{align*}
    \frac{d}{d \alpha}\gH_\alpha(\alpha)
        &= - \alpha A''(\alpha) - \frac{d}{d\alpha}\E_{p_\alpha}[k(x)]\\
        &= - \alpha A''(\alpha) - \E_\alpha[k(x)(T(x) - A'(\alpha)]\\
        &= - \alpha \Var_{p_\alpha}[T(x)] - \Cov_{p_\alpha}[k(x), T(x)]
\end{align*}
where we use the fact that the $n$th derivative of $A(\alpha)$ give the $n$ central moment, i.e. $A'(\alpha) = \E_{p_\alpha}[T(x)]$ and $A''(\alpha) = \Var_{p_\alpha}[T(x)]$.
The derivative of $\alpha = 0$ is
\begin{align*}
    \frac{d}{d \alpha}\gH_\alpha(0)
        &= - \Cov_{p_0}[k(x), T(x)]\\
        &= - \Cov_p[\log p(x), \log q(x)]
\end{align*}
which is negative by assumption.
Because the derivative at $\alpha = 0$ is negative, then there exists a constant $a > 0$ such that for all $\alpha \in [-a, 0]$, $\gH_\alpha(\alpha) > \gH_\alpha(0) = \gH(p)$.
\end{proof}
The paper applies \label{lemma:covariance-general} to the case where $q = \pg$ and $p=\pstate$.
When we take $N \rightarrow \infty$, we have that $\pskewed$ corresponds to $p_\alpha$ above.

\subsection{Simple Case Proof}\label{sec:simple-case-proof}
We prove the convergence directly for the (even more) simplified case when $\papprox = \pspgt$ using a similar technique:
\begin{lemma}\label{lemma:all-equal}
Assume the set $\Imgs$ has finite volume so that its uniform distribution $\U$ is well defined and has finite entropy.
Given any distribution $p(\st)$ whose support is $\Imgs$, recursively define $p_t$ with $p_1 = p$ and
\begin{align*}
    p_{t+1}(\st)
        &= \frac{1}{Z_\alpha^t} p_t(\st)^{\alpha}, \quad \forall \st \in \Imgs
\end{align*}
where $Z_\alpha^t$ is the normalizing constant and $\alpha \in [0, 1)$.

The sequence $(p_1, p_2, \dots)$ converges to $\U$, the uniform distribution $\Imgs$.
\end{lemma}
\begin{proof}
If $\alpha = 0$, then $p_2$ (and all subsequent distributions) will clearly be the uniform distribution.
We now study the case where $\alpha \in (0, 1)$.

At each iteration $t$, define the one-dimensional exponential family $\{\ptt : \theta \in [0, 1]\}$ where $\ptt$ is
\begin{align*}
    \ptt(\st) = e^{\theta T(\st) - A(\theta) + k(\st)}
\end{align*}
with log carrier density $k(\st) = 0$, natural parameter $\theta$, sufficient statistic $T(\st) = \log p_t(\st)$, and log-normalizer $A(\theta) = \int_{\Imgs} e^{\theta T(\st)} d\st$.
As shown in \cite{nielsen2010entropies}, the entropy of a distribution from a one-dimensional exponential family with parameter $\theta$ is given by:
\begin{align*}
    \htt(\theta) \triangleq \gH(\ptt) = A(\theta) - \theta A'(\theta)
\end{align*}
The derivative with respect to $\theta$ is then
\begin{align}\nonumber
    \frac{d}{d \theta}d\htt(\theta)
        &= - \theta A''(\theta)\\\nonumber
        &= - \theta \Var_{\st \sim \ptt}[T(\st)]\\\label{eq:easy-variance-case}
        &= - \theta \Var_{\st \sim \ptt}[\log p_t(\st)]\\\nonumber
        &\leq 0
\end{align}
where we use the fact that the $n$th derivative of $A(\theta)$ is the $n$ central moment, i.e. $A''(\theta) = \Var_{\st \sim \ptt}[T(\st)]$.
Since variance is always non-negative, this means the entropy is monotonically decreasing with $\theta$.
Note that $p_{t+1}$ is a member of this exponential family, with parameter $\theta = \alpha \in (0, 1)$.
So
\begin{align*}
    \gH(p_{t+1}) = \htt(\alpha) \geq \htt(1) = \gH(p_t)
\end{align*}
which implies
\begin{align*}
    \gH(p_1) \leq \gH(p_2) \leq \dots.
\end{align*}
This monotonically increasing sequence is upper bounded by the entropy of the uniform distribution, and so this sequence must converge.

The sequence can only converge if $\frac{d}{d \theta} \htt(\theta)$ converges to zero.
However, because $\alpha$ is bounded away from $0$, \Eqref{eq:easy-variance-case} states that this can only happen if
\begin{align}\label{eq:easy-logprob-becomes-constant}
    \Var_{\st \sim \ptt}[\log p_t(\st)] \rightarrow 0.
\end{align}
Because $p_t$ has full support, then so does $\ptt$.
Thus, \Eqref{eq:easy-logprob-becomes-constant} is only true if $\log p_t(\st)$ converges to a constant, i.e. $p_t$ converges to the uniform distribution.
\end{proof}


\section{Additional Experiments}\label{sec:add-exps}

\subsection{Sensitivity Analysis}\label{sensitivity}
\paragraph{Sensitivity to RL Algorithm}
In our experiments, we combined \METHOD with soft actor critic (SAC)~\citep{haarnoja2018sacapp}.
We conduct a set of experiments to test whether \METHOD may be used with other RL algorithms for training the goal-conditioned policy.
To that end, we replaced SAC with twin delayed deep deterministic policy gradient (TD3)~\citep{fujimoto2018td3} and ran the same \METHOD experiments on Visual Door, Visual Pusher, and Visual Pickup.
In \autoref{fig:rl-sweep}, we see that \METHOD performs consistently well with both SAC and TD3, demonstrating that \METHOD is beneficial across multiple RL algorithms.
\begin{figure}
    \begin{subfigure}{\linewidth}
      \centering
      \includegraphics[width=\linewidth]{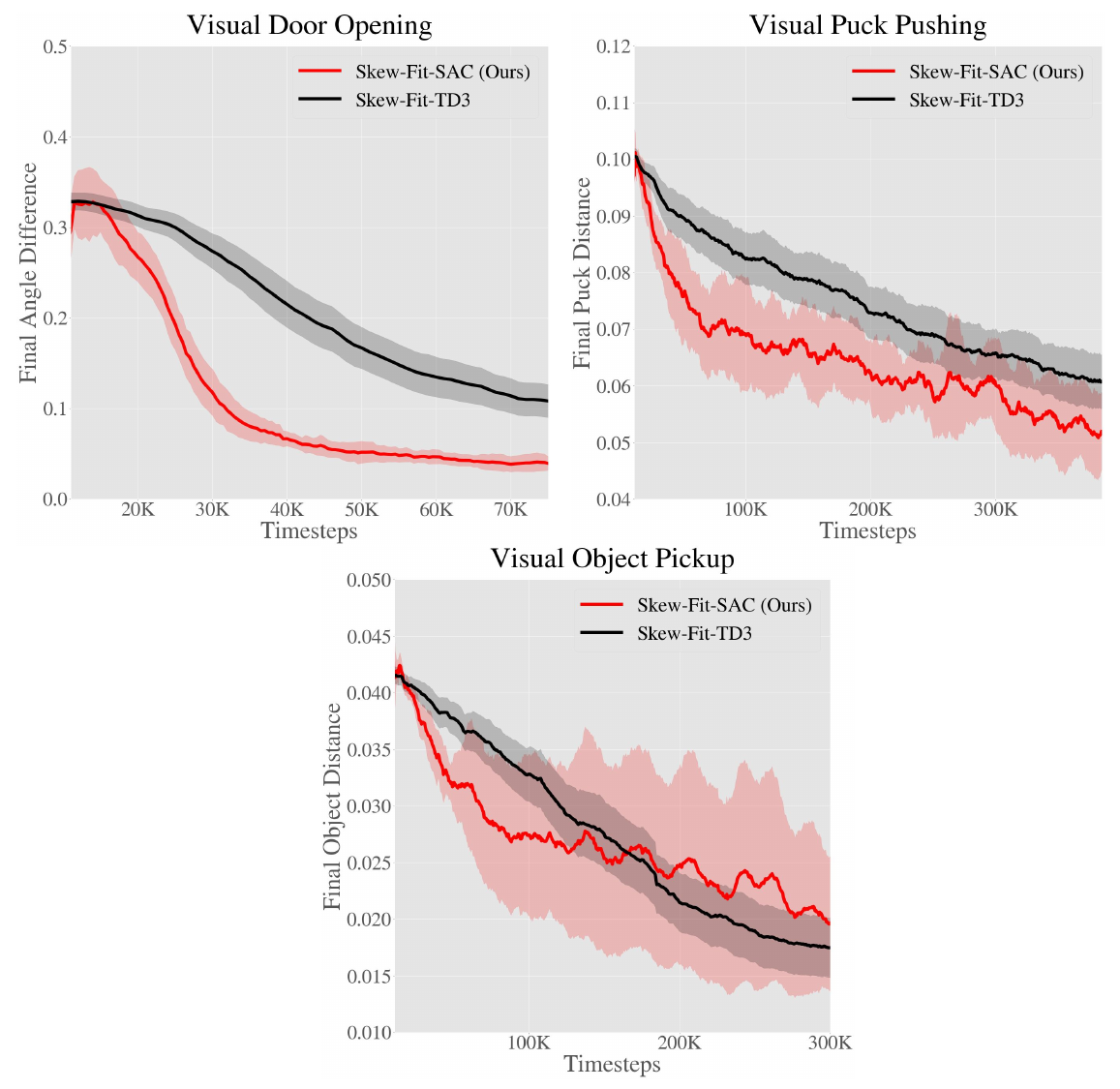}
      \label{fig:door-rl-sweep}
    \end{subfigure}
    \\
  \vspace{-.1in}
  \fcaption{
  We compare using SAC~\citep{haarnoja2018sacapp} and TD3~\citep{fujimoto2018td3} as the underlying RL algorithm on Visual Door, Visual Pusher and Visual Pickup.
  We see that \METHOD works consistently well with both SAC and TD3, demonstrating that \METHOD may be used with various RL algorithms.
  For the experiments presented in \autoref{sec:experiments}, we used SAC.
  }
  \vspace{-0.05 cm}
  \label{fig:rl-sweep}
\end{figure}

\paragraph{Sensitivity to $\alpha$ Hyperparameter}
We study the sensitivity of the $\alpha$ hyperparameter by testing values of \mbox{$\alpha \in [-1, -0.75, -0.5, -0.25, 0]$} on the Visual Door and Visual Pusher task.
The results are included in \Figref{fig:alpha-sweep} and shows that our method is  robust to different parameters of $\alpha$, particularly for the more challenging Visual Pusher task.
Also, the method consistently outperform $\alpha=0$, which is equivalent to sampling uniformly from the replay buffer.
\begin{figure}[!ht]
  \centering
  \includegraphics[width=1\linewidth]{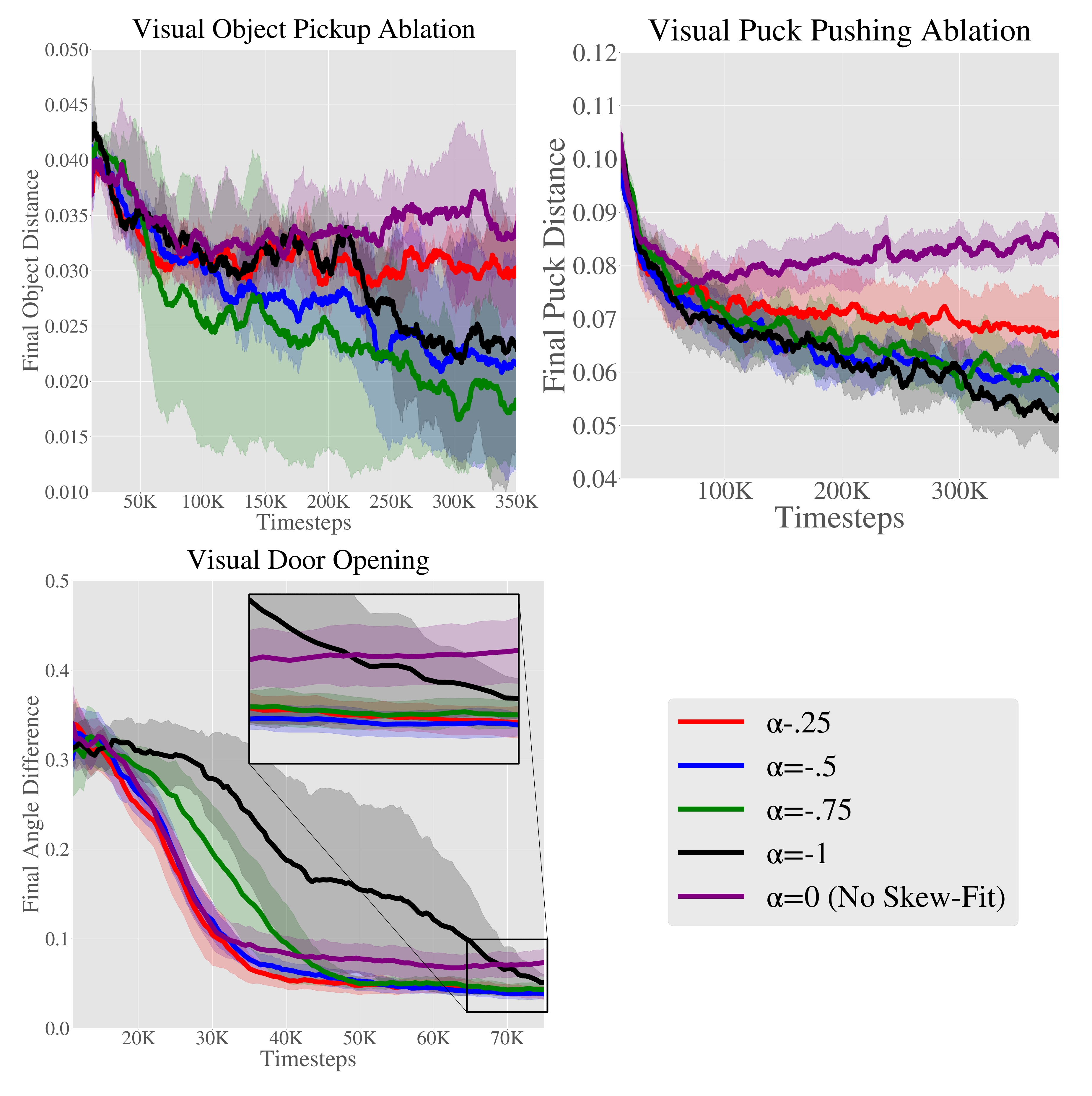}
  \vspace{-.05in}
  \label{fig:door-alpha-sweep}
  \vspace{-.1in}
  \fcaption{
  We sweep different values of $\alpha$ on Visual Door, Visual Pusher and Visual Pickup.
  \METHOD helps the final performance on the Visual Door task, and outperforms No Skew-Fit ($\alpha=0$) as seen in the zoomed in version of the plot.
  In the more challenging Visual Pusher task, we see that \METHOD consistently helps and halves the final distance.
  Similarly, we observe that \METHOD consistently outperforms No Skew-fit on Visual Pickup.
  Note that alpha=-1 is not always the optimal setting for each environment, but outperforms $\alpha=0$ in each case in terms of final performance.
  }
  \vspace{-0.05 cm}
  \label{fig:alpha-sweep}
\end{figure}

 \subsection{Variance Ablation} \label{sec:analysis-variance}
 \begin{figure}[H]
     \vspace{0.1in}
         \centering
         \includegraphics[width=0.8\linewidth]{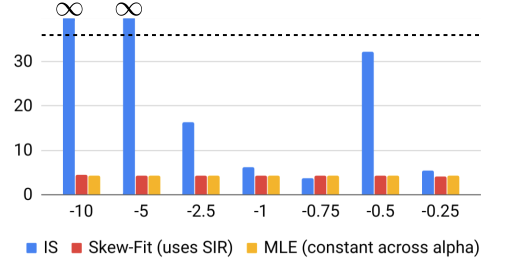}
         \fcaption{Gradient variance averaged across parameters in last epoch of training VAEs. Values of $\alpha$ less than $-1$ are numerically unstable for importance sampling (IS), but not for Skew-Fit.}
         \label{fig:grad-var}
     \label{fig:ul-variance-results}
 \end{figure}
We measure the gradient variance of training a VAE on an unbalanced Visual Door image dataset with \METHOD vs \METHOD with importance sampling (IS) vs no \METHOD (labeled MLE). We construct the imbalanced dataset by rolling out a random policy in the environment and collecting the visual observations. Most of the images contained the door in a closed position; in a few, the door was opened.
 In \autoref{fig:ul-variance-results}, we see that the gradient variance for Skew-Fit with IS is catastrophically large for large values of $\alpha$.
 In contrast, for Skew-Fit with SIR, which is what we use in practice, the variance is relatively similar to that of MLE. Additionally we trained three VAE's, one with MLE on a uniform dataset of valid door opening images, one with Skew-Fit on the unbalanced dataset from above, and one with MLE on the same unbalanced dataset. As expected, the VAE that has access to the uniform dataset gets the lowest negative log likelihood score.
 This is the oracle method, since in practice we would only have access to imbalanced data.
 As shown in \autoref{table:ll-ablation}, \METHOD considerably outperforms MLE, getting a much closer to oracle log likelihood score.
 \begin{table}
         \centering
         \begin{tabular}{|l|l|}
         \hline
             {\footnotesize \textbf{Method}}                 & {\footnotesize \textbf{NLL}}   \\\hline
             {\footnotesize MLE on uniform (oracle)} & {\footnotesize 20175.4}               \\\hline
             {\footnotesize Skew-Fit on unbalanced}           & {\footnotesize 20175.9}              \\\hline
             {\footnotesize MLE on unbalanced}                    & {\footnotesize 20178.03} \\\hline
         \end{tabular}
         \fcaption{Despite training on a unbalanced Visual Door dataset (see Figure 7 of paper), the negative log-likelihood (NLL) of Skew-Fit evaluated on a uniform dataset matches that of a VAE trained on a uniform dataset.}
         \label{table:ll-ablation}
 \end{table}

\subsection{Goal and Performance Visualization}\label{sec:vae-dump}
We visualize the goals sampled from \METHOD as well as those sampled when using the prior method, RIG~\citep{nair2018rig}.
As shown in \autoref{fig:vae_dump} and \autoref{fig:vae_dump_real}, the generative model $\pg$ results in much more diverse samples when trained with \METHOD.
We we see in \autoref{fig:example_rollouts}, this results in a policy that more consistently reaches the goal image.
\begin{figure*}
    \centering
    \includegraphics[width=0.9\textwidth]{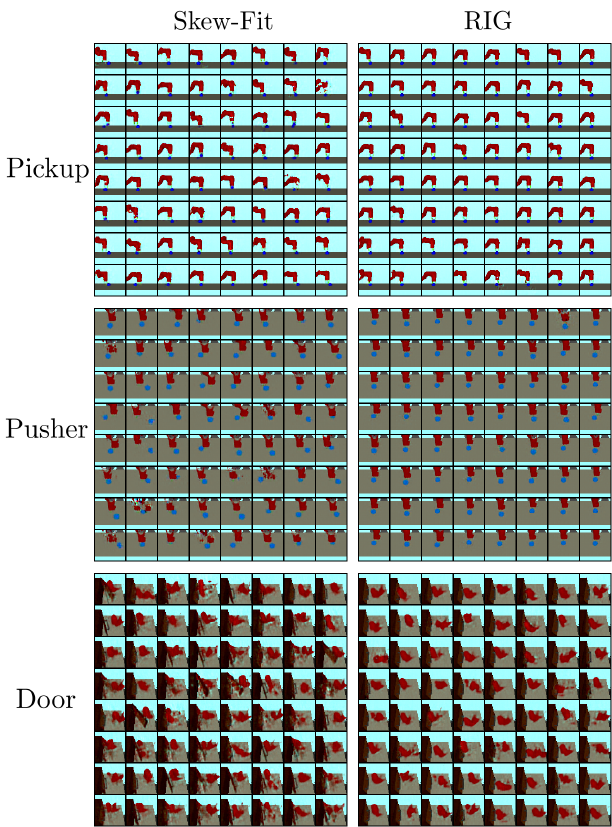}
    \fcaption{
Proposed goals from the VAE for RIG and with \METHOD on the \textit{Visual Pickup}, \textit{Visual Pusher}, and \textit{Visual Door} environments. Standard RIG produces goals where the door is closed and the object and puck is in the same position, while RIG + \METHOD proposes goals with varied puck positions, occasional object goals in the air, and both open and closed door angles.
    }
    \label{fig:vae_dump}
\end{figure*}

\begin{figure*}
    \centering
    \includegraphics[width=0.25\textheight]{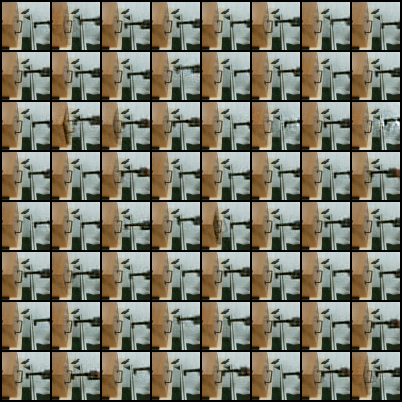}
    \hspace{.3 in}
    \includegraphics[width=0.25\textheight]{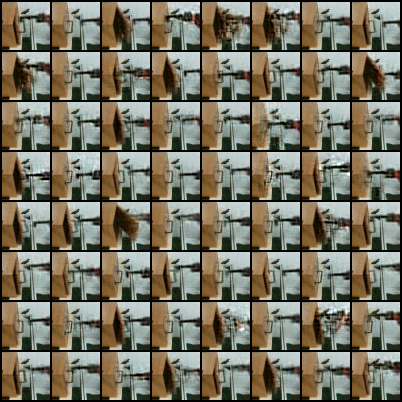}
    \fcaption{
Proposed goals from the VAE for RIG (left) and with RIG + \METHOD (right) on the \textit{Real World Visual Door} environment. Standard RIG produces goals where the door is closed  while RIG + \METHOD proposes goals with both open and closed door angles.
    }
    \label{fig:vae_dump_real}
\end{figure*}
\begin{figure*}
    \centering
    \includegraphics[width=0.9\textwidth]{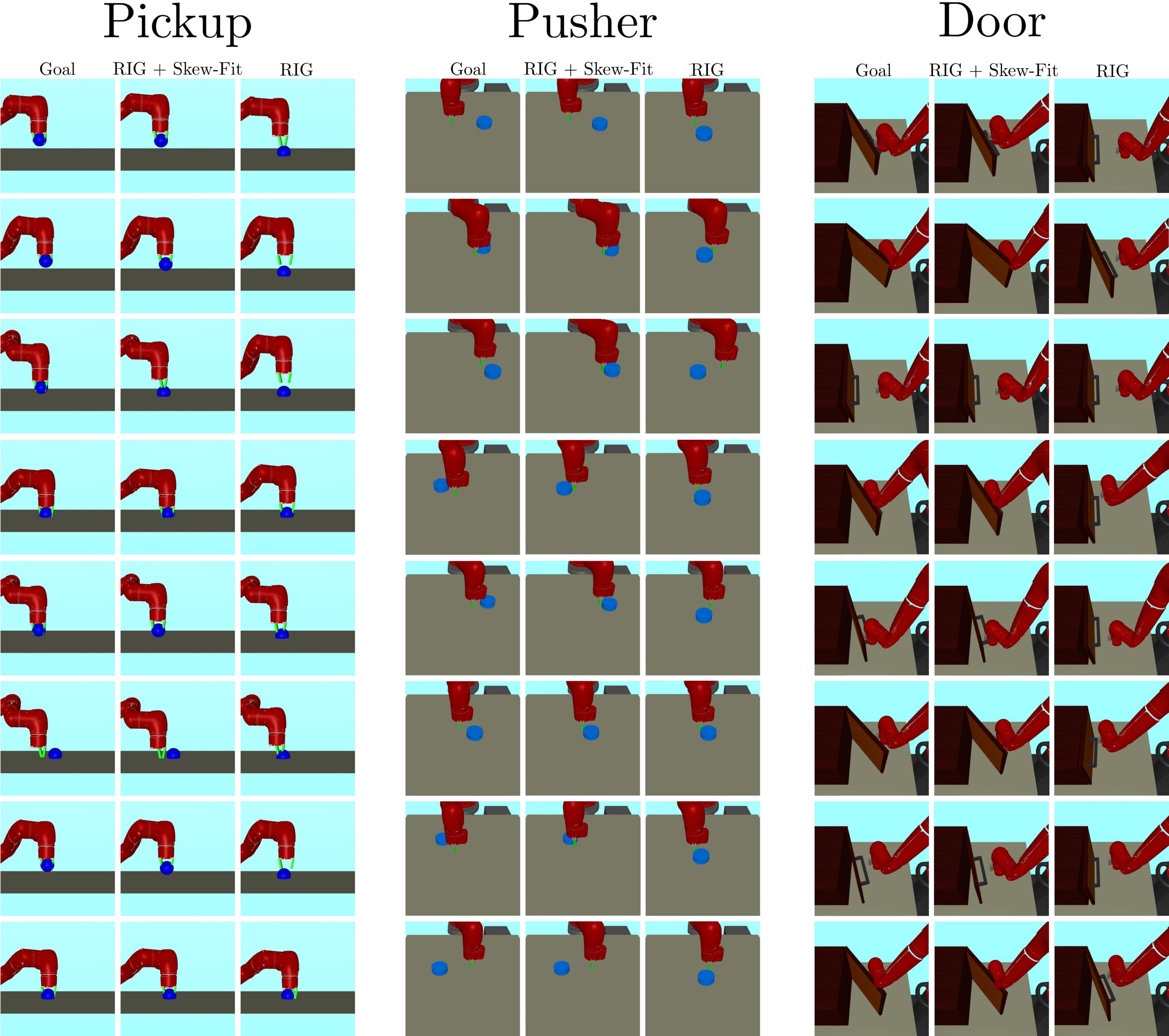}
    \fcaption{
Example reached goals by \METHOD and RIG. The first column of each environment section specifies the target goal while the second and third columns show reached goals by \METHOD and RIG. Both methods learn how to reach goals close to the initial position, but only \METHOD learns to reach the more difficult goals.
    }
    \label{fig:example_rollouts}
\end{figure*}

\section{Implementation Details}\label{sec:implementation-details}

\subsection{Likelihood Estimation using $\beta$-VAE}\label{sec:likelihood-estimation-vae}
We estimate the density under the VAE by using a sample-wise approximation to the marginal over $x$ estimated using importance sampling:
\begin{align*}
    \pgt(x) &= \mathbb E_{z \sim q_{\theta_t}(z|x)} \left[\frac{p(z)}{q_{\theta_t}(z|x)} p_{\psi_t}(x \mid z) \right]  \\
    &\approx \frac{1}{N} \sum_{i=1}^{N} \left[\frac{p(z)}{q_{\theta_t}(z|x)} p_{\psi_t}(x \mid z) \right].
\end{align*}
where $q_{\theta}$ is the encoder, $p_\psi$ is the decoder, and $p(z)$ is the prior, which in this case is unit Gaussian.
We found that sampling $N=10$ latents for estimating the density worked well in practice.

\subsection{Oracle 2D Navigation Experiments}\label{sec:2d-details}
We initialize the VAE to the bottom left corner of the environment for \textit{Four Rooms}.
Both the encoder and decoder have 2 hidden layers with [400, 300] units,  ReLU hidden activations, and no output activations.
The VAE has a latent dimension of $8$ and a Gaussian decoder trained with a fixed variance, batch size of $256$, and $1000$ batches at each iteration. The VAE is trained on the exploration data buffer every 1000 rollouts.

\subsection{Implementation of SAC and Prior Work}\label{sec:prior-work-implementation}
For all experiments, we trained the goal-conditioned policy using soft actor critic (SAC)~\citep{haarnoja2018sacapp}.
To make the method goal-conditioned, we concatenate the target XY-goal to the state vector.
During training, we retroactively relabel the goals~\citep{kaelbling1993goals,andrychowicz2017her} by sampling from the goal distribution with probabilty $0.5$.
Note that the original RIG~\cite{nair2018rig} paper used TD3~\cite{fujimoto2018td3}, which we also replaced with SAC in our implementation of RIG.
We found that maximum entropy policies in general improved the performance of RIG, and that we did not need to add noise on top of the stochastic policy's noise.
In the prior RIG method, the VAE was pre-trained on a uniform sampling of images from the state space of each environment.
In order to ensure a fair comparison to Skew-Fit, we forego pre-training and instead train the VAE alongside RL, using the variant described in the RIG paper.
For our RL network architectures and training scheme, we use fully connected networks for the policy, Q-function and value networks with two hidden layers of size $400$ and $300$ each.
We also delay training any of these networks for $10000$ time steps in order to collect sufficient data for the replay buffer as well as to ensure the latent space of the VAE is relatively stable (since we continuously train the VAE concurrently with RL training).
As in RIG, we train a goal-conditioned value functions~\cite{schaul2015uva} using hindsight experience replay~\cite{andrychowicz2017her}, relabelling $50\%$ of exploration goals as goals sampled from the VAE prior $\mathcal{N} (0, 1)$ and $30\%$ from future goals in the trajectory.

For our implementation of \citep{hazan2019provably}, we trained the policies with the reward
\begin{align*}
    r(s) = r_\text{Skew-Fit}(s) + \lambda \cdot r_\text{Hazan et al.}(s)
\end{align*}

For $r_\text{Hazan et al.}$, we use the reward described in Section 5.2 of \citet{hazan2019provably}, which requires an estimated likelihood of the state.
To compute these likelihood, we use the same method as in Skew-Fit (see \autoref{sec:likelihood-estimation-vae}).
With 3 seeds each, we tuned $\lambda$ across values $[100, 10, 1, 0.1, 0.01, 0.001]$ for the door task, but all values performed poorly.
For the pushing and picking tasks, we tested values across $[1, 0.1, 0.01, 0.001, 0.0001]$ and found that 0.1 and 0.01 performed best for each task, respectively.

\subsection{RIG with \METHOD Summary}\label{sec:rig-and-full-method}
\autoref{alg:rig-and-full-method-details} provides detailed pseudo-code for how we combined our method with RIG.
Steps that were removed from the base RIG algorithm are highlighted in blue and steps that were added are highlighted in red.
The main differences between the two are (1) not needing to pre-train the $\beta$-VAE, (2) sampling exploration goals from the buffer using $\pskewed$ instead of the VAE prior, (3) relabeling with replay buffer goals sampled using $\pskewed$ instead of from the VAE prior, and (4) training the VAE on replay buffer data data sampled using $\pskewed$ instead of uniformly.

\subsection{Vision-Based Continuous Control Experiments}\label{sec:vision-based-appendix}
In our experiments, we use an image size of 48x48.
For our VAE architecture, we use a modified version of the architecture used in the original RIG paper ~\cite{nair2018rig}.
Our VAE has three convolutional layers with kernel sizes: 5x5, 3x3, and 3x3, number of output filters: 16, 32, and 64 and strides: 3, 2, and 2.
We then have a fully connected layer with the latent dimension number of units, and then reverse the architecture with de-convolution layers.
We vary the latent dimension of the VAE, the $\beta$ term of the VAE and the $\alpha$ term for \METHOD based on the environment.
Additionally, we vary the training schedule of the VAE based on the environment. See the table at the end of the appendix for more details.
Our VAE has a Gaussian decoder with identity variance, meaning that we train the decoder with a mean-squared error loss.

When training the VAE alongside RL, we found the following three schedules to be effective for different environments:

\begin{enumerate}
    \item For first $5K$ steps: Train VAE using standard MLE training every $500$ time steps for $1000$ batches. After that, train VAE using \METHOD every $500$ time steps for $200$ batches.
    \item For first $5K$ steps: Train VAE using standard MLE training every $500$ time steps for $1000$ batches. For the next $45K$ steps, train VAE using \METHOD every $500$ steps for $200$ batches. After that, train VAE using \METHOD every $1000$ time steps for $200$ batches.
    \item For first $40K$ steps: Train VAE using standard MLE training every $4000$ time steps for $1000$ batches. Afterwards, train VAE using \METHOD every $4000$ time steps for $200$ batches.
\end{enumerate}

We found that initially training the VAE without \METHOD improved the stability of the algorithm.
This is due to the fact that density estimates under the VAE are constantly changing and inaccurate during the early phases of training. Therefore, it made little sense to use those estimates to prioritize goals early on in training.
Instead, we simply train using MLE training for the first $5K$ timesteps, and after that we perform \METHOD according to the VAE schedules above. Table \ref{table:general-hyperparams} lists the hyper-parameters that were shared across the continuous control experiments. Table \ref{table:env-hyperparams} lists hyper-parameters specific to each environment. Additionally, \autoref{sec:rig-and-full-method} discusses the combined RIG + Skew-Fit algorithm.

\begin{table*}
    \centering
    \begin{tabular}{c|c|c}
    \hline
    \textbf{Hyper-parameter} & \textbf{Value} & \textbf{Comments}\\
    \hline
    \# training batches per time step & $2$ & Marginal improvements after $2$\\
    Exploration Noise & None (SAC policy is stochastic) & Did not tune\\
    RL Batch Size & $1024$ & smaller batch sizes work as well\\
    VAE Batch Size &  $64$ & Did not tune \\
    Discount Factor & $0.99$ & Did not tune\\
    Reward Scaling & $1$ & Did not tune\\
    Policy Hidden Sizes & $[400, 300]$ & Did not tune\\
    Policy Hidden Activation & ReLU & Did not tune\\
    Q-Function Hidden Sizes & $[400, 300]$ & Did not tune\\
    Q-Function Hidden Activation & ReLU & Did not tune\\
    Replay Buffer Size & $100000$ & Did not tune\\
    Number of Latents for Estimating Density ($N$) & $10$ & Marginal improvements beyond $10$\\
    \hline
    \end{tabular}
\fcaption{General hyper-parameters used for all \textit{visual} experiments.}
\label{table:general-hyperparams}
\end{table*}

\begin{table*}
    \centering
    \begin{tabular}{c|c|c|c|c}
    \hline
    \textbf{Hyper-parameter} & \textbf{Visual Pusher} & \textbf{Visual Door} & \textbf{Visual Pickup} & \textbf{Real World Visual Door}\\
    \hline
    Path Length & $50$& $100$ & $50$ & $100$ \\
    $\beta$ for $\beta$-VAE & $20$ & $20$ & $30$ & $60$ \\
    Latent Dimension Size & $4$ & $16$ & $16$ & $16$ \\
    $\alpha$ for Skew-Fit & $-1$ & $-1/2$ & $-1$ & $-1/2$ \\
    VAE Training Schedule & $2$ & $1$ & $2$ & $1$ \\
    Sample Goals From & $\pg$ & $\pskewed$ & $\pskewed$ & $\pskewed$ \\
    \hline
    \end{tabular}
\fcaption{Environment specific hyper-parameters for the \textit{visual} experiments}
\label{table:env-hyperparams}
\end{table*}

\begin{table*}
    \centering
    \begin{tabular}{c|c}
    \hline
    \textbf{Hyper-parameter} & \textbf{Value} \\
    \hline
    \# training batches per time step & $.25$\\
    Exploration Noise & None (SAC policy is stochastic) \\
    RL Batch Size & $512$\\
    VAE Batch Size &  $64$\\
    Discount Factor & $\frac{299}{300}$\\
    Reward Scaling & $10$\\
    Path length & $300$\\
    Policy Hidden Sizes & $[400, 300]$\\
    Policy Hidden Activation & ReLU\\
    Q-Function Hidden Sizes & $[400, 300]$\\
    Q-Function Hidden Activation & ReLU\\
    Replay Buffer Size & $1000000$\\
    Number of Latents for Estimating Density ($N$) & $10$\\
    $\beta$ for $\beta$-VAE & $10$ \\
    Latent Dimension Size & $2$ \\
    $\alpha$ for Skew-Fit & $-2.5$ \\
    VAE Training Schedule & $3$ \\
    Sample Goals From & $\pskewed$ \\
    \hline
    \end{tabular}
\fcaption{Hyper-parameters used for the \textit{ant} experiment.}
\label{table:ant-hyperparams}
\end{table*}

\begin{algorithm}
   	\footnotesize
   	\fcaption{RIG and RIG + Skew-Fit. Blue text denotes RIG specific steps and red text denotes RIG + Skew-Fit specific steps}
   	\label{alg:rig-and-full-method-details}
   	\begin{algorithmic}[1]
    \REQUIRE $\beta$-VAE mean encoder $q_\phi$, $\beta$-VAE decoder $p_\psi$, policy $\pi_\theta$, goal-conditioned value function $Q_w$, skew parameter $\alpha$, VAE Training Schedule.
    \STATE \textcolor{blue}{Collect $\mathcal D = \{s^{(i)}\}$ using random initial policy.}
    \STATE \textcolor{blue}{Train $\beta$-VAE on data uniformly sampled from $\mathcal D$}.
    \STATE \textcolor{blue}{Fit prior $p(z)$ to latent encodings $\{\mu_\phi(s^{(i)})\}$.}
    \FOR{$n=0,...,N-1$ episodes}
        \STATE \textcolor{blue}{Sample latent goal from prior $z_g \sim p(z)$}.
        \STATE \textcolor{red}{Sample state $s_g \sim \pskewedn$ and encode $z_g = q_\phi(s_g)$ if $\mathcal R$ is nonempty. Otherwise sample $z_g \sim p(z)$}
        \STATE Sample initial state $s_0$ from the environment.
        \FOR{$t=0,...,H -1$ steps}
            \STATE Get action $a_t \sim \pi_\theta(q_\phi(s_t), z_g)$.
            \STATE Get next state $s_{t+1} \sim p(\cdot \mid s_t, a_t)$.
            \STATE Store $(s_t, a_t, s_{t+1}, z_g)$ into replay buffer $\mathcal R$.
            \STATE Sample transition $(s, a, s', z_g) \sim \mathcal R$.
            \STATE Encode $z = q_\phi(s), z' = q_\phi(s')$.
            \STATE \textcolor{blue}{(Probability $0.5$) replace $z_g$ with $z_g' \sim p(z)$.}
            \STATE \textcolor{red}{(Probability $0.5$) replace $z_g$ with $q_\phi(s'')$ where $s'' \sim \pskewedn$.}
            \STATE Compute new reward $r = -||z' - z_g||$.
            \STATE Minimize Bellman Error using $(z, a, z', z_g, r)$.
        \ENDFOR
        \FOR{$t=0,...,H -1$ steps}
            \FOR{$i=0,...,k-1$ steps}
                \STATE Sample future state $s_{h_i}$, $t < h_i \leq H-1$.
                \STATE Store $(s_t, a_t, s_{t+1}, q_\phi(s_{h_i}))$ into $\mathcal R$.
            \ENDFOR
        \ENDFOR
        \STATE \textcolor{red}{Construct skewed replay buffer distribution $\pskewednn$ using data from $\mathcal R$ with \Eqref{eq:pskew-defn}.}
        \IF {total steps $< 5000$}
            \STATE Fine-tune $\beta$-VAE on data uniformly sampled from $\mathcal R$ according to VAE Training Schedule.
        \ELSE
            \STATE \textcolor{blue}{Fine-tune $\beta$-VAE on data uniformly sampled from $\mathcal R$ according to VAE Training Schedule.}
            \STATE \textcolor{red}{Fine-tune $\beta$-VAE on data sampled from $\pskewednn$ according to VAE Training Schedule.}
        \ENDIF
    \ENDFOR
   	\end{algorithmic}
\end{algorithm}

\section{Environment Details}\label{sec:environment-details}
\textit{Four Rooms}: A 20 x 20 2D pointmass environment in the shape of four rooms~\citep{sutton1999between}. The observation is the 2D position of the agent, and the agent must specify a target 2D position as the action.
The dynamics of the environment are the following:
first, the agent is teleported to the target position, specified by the action.
Then a Gaussian change in position with mean $0$ and standard deviation $0.0605$ is applied\footnote{In the main paper, we rounded this to $0.06$, but this difference does not matter.}.
If the action would result in the agent moving through or into a wall, then the agent will be stopped at the wall instead.

\textit{Ant}: A MuJoCo~\citep{todorov12mujoco} ant environment. The observation is a 3D position and velocities, orientation, joint angles, and velocity of the joint angles of the ant (8 total). The observation space is 29 dimensions.
The agent controls the ant through the joints, which is 8 dimensions. The goal is a target 2D position, and the reward is the negative Euclidean distance between the achieved 2D position and target 2D position.

\textit{Visual Pusher}: A MuJoCo environment with a 7-DoF Sawyer arm and a small puck on a table that the arm must push to a target position.
The agent controls the arm by commanding $x,y$ position for the end effector (EE).
The underlying state is the EE position, $e$ and puck position $p$.
The evaluation metric is the distance between the goal and final puck positions. The hand goal/state space is a 10x10 cm$^2$ box and the puck goal/state space is a 30x20 cm$^2$ box. Both the hand and puck spaces are centered around the origin. The action space ranges in the interval $[-1, 1]$ in the x and y dimensions.

\textit{Visual Door}: A MuJoCo environment with a 7-DoF Sawyer arm and a door on a table that the arm must pull open to a target angle.
Control is the same as in \textit{Visual Pusher}.
The evaluation metric is the distance between the goal and final door angle, measured in radians.
In this environment, we do not reset the position of the hand or door at the end of each trajectory. The state/goal space is a 5x20x15 cm$^3$ box in the $x, y, z$ dimension respectively for the arm and an angle between $[0, .83]$ radians. The action space ranges in the interval $[-1, 1]$ in the x, y and z dimensions.

\textit{Visual Pickup}: A MuJoCo environment with the same robot as \textit{Visual Pusher}, but now with a different object. The object is cube-shaped, but a larger intangible sphere is overlaid on top so that it is easier for the agent to see.
Moreover, the robot is constrained to move in 2 dimension: it only controls the $y, z$ arm positions.
The $x$ position of both the arm and the object is fixed.
The evaluation metric is the distance between the goal and final object position.
For the purpose of evaluation, $75\%$ of the goals have the object in the air and $25\%$ have the object on the ground. The state/goal space for both the object and the arm is 10cm in the $y$ dimension and 13cm in the $z$ dimension. The action space ranges in the interval $[-1, 1]$ in the $y$ and $z$ dimensions.

\textit{Real World Visual Door}:
A Rethink Sawyer Robot with a door on a table.
The arm must pull the door open to a target angle.
The agent controls the arm by commanding the $x,y,z$ velocity of the EE.
Our controller commands actions at a rate of up to 10Hz with the scale of actions ranging up to 1cm in magnitude.
The underlying state and goal is the same as in \textit{Visual Door}.
Again we do not reset the position of the hand or door at the end of each trajectory.
We obtain images using a Kinect Sensor.
The state/goal space for the environment is a 10x10x10 cm$^3$ box. The action space ranges in the interval $[-1, 1]$  (in cm) in the x, y and z dimensions. The door angle lies in the range $[0, 45]$ degrees.

\section{Goal-Conditioned Reinforcement Learning Minimizes $\gH(\G \mid \SF)$}\label{sec:analysis-appendix}
Some goal-conditioned RL methods such as \citet{wardefarley2018discern,nair2018rig} present methods for minimizing a lower bound for $\gH(\G \mid \SF)$, by approximating $\log p(\G \mid \SF)$ and using it as the reward.
Other goal-conditioned RL methods
~\citep{kaelbling1993goals,lillicrap2015continuous, schaul2015uva, andrychowicz2017her, pong2018tdm,florensa2018selfsupervised}
are not developed with the intention of minimizing the conditional entropy $\gH(\G \mid \SF)$.
Nevertheless, one can see that goal-conditioned RL generally minimizes $\gH(\G \mid \SF)$ by noting that the optimal goal-conditioned policy will deterministically reach the goal.
The corresponding conditional entropy of the goal given the state, $\gH(\G \mid \SF)$, would be zero, since given the current state, there would be no uncertainty over the goal (the goal must have been the current state since the policy is optimal).
So, the objective of goal-conditioned RL can be interpreted as finding a policy such that $\gH(\G \mid \SF) = 0$.
Since zero is the minimum value of $\gH(\G \mid \SF)$, then goal-conditioned RL can be interpreted as minimizing $\gH(\G \mid \SF)$.

%% file: main.bbl
\begin{thebibliography}{45}
\providecommand{\natexlab}[1]{#1}
\providecommand{\url}[1]{\texttt{#1}}
\expandafter\ifx\csname urlstyle\endcsname\relax
  \providecommand{\doi}[1]{doi: #1}\else
  \providecommand{\doi}{doi: \begingroup \urlstyle{rm}\Url}\fi

\bibitem[Andrychowicz et~al.(2017)Andrychowicz, Wolski, Ray, Schneider, Fong,
  Welinder, Mcgrew, Tobin, Abbeel, and Zaremba]{andrychowicz2017her}
Andrychowicz, M., Wolski, F., Ray, A., Schneider, J., Fong, R., Welinder, P.,
  Mcgrew, B., Tobin, J., Abbeel, P., and Zaremba, W.
\newblock {Hindsight Experience Replay}.
\newblock In \emph{Advances in Neural Information Processing Systems (NIPS)},
  2017.

\bibitem[Baranes \& Oudeyer(2012)Baranes and Oudeyer]{baranes2012}
Baranes, A. and Oudeyer, P.-Y.
\newblock {Active Learning of Inverse Models with Intrinsically Motivated Goal
  Exploration in Robots}.
\newblock \emph{Robotics and Autonomous Systems}, 61\penalty0 (1):\penalty0
  49--73, 2012.
\newblock \doi{10.1016/j.robot.2012.05.008}.

\bibitem[Barber \& Agakov(2004)Barber and Agakov]{barber2004information}
Barber, D. and Agakov, F.~V.
\newblock Information maximization in noisy channels: A variational approach.
\newblock In \emph{Advances in Neural Information Processing Systems}, pp.\
  201--208, 2004.

\bibitem[Bellemare et~al.(2016)Bellemare, Srinivasan, Ostrovski, Schaul,
  Saxton, and Munos]{bellemare2016unifying}
Bellemare, M., Srinivasan, S., Ostrovski, G., Schaul, T., Saxton, D., and
  Munos, R.
\newblock {Unifying count-based exploration and intrinsic motivation}.
\newblock In \emph{Advances in Neural Information Processing Systems (NIPS)},
  pp.\  1471--1479, 2016.

\bibitem[Burda et~al.(2018)Burda, Edwards, Storkey, and
  Klimov]{burda2018exploration}
Burda, Y., Edwards, H., Storkey, A., and Klimov, O.
\newblock Exploration by random network distillation.
\newblock \emph{arXiv preprint arXiv:1810.12894}, 2018.

\bibitem[Burda et~al.(2019)Burda, Edwards, Pathak, Storkey, Darrell, and
  Efros]{burda2018large}
Burda, Y., Edwards, H., Pathak, D., Storkey, A., Darrell, T., and Efros, A.~A.
\newblock Large-scale study of curiosity-driven learning.
\newblock In \emph{International Conference on Learning Representations
  (ICLR)}, 2019.

\bibitem[Chebotar et~al.(2017)Chebotar, Kalakrishnan, Yahya, Li, Schaal, and
  Levine]{chebotar2017path}
Chebotar, Y., Kalakrishnan, M., Yahya, A., Li, A., Schaal, S., and Levine, S.
\newblock Path integral guided policy search.
\newblock In \emph{2017 IEEE International Conference on Robotics and
  Automation (ICRA)}, pp.\  3381--3388. IEEE, 2017.

\bibitem[Chentanez et~al.(2005)Chentanez, Barto, and
  Singh]{chentanez2005intrinsically}
Chentanez, N., Barto, A.~G., and Singh, S.~P.
\newblock Intrinsically motivated reinforcement learning.
\newblock In \emph{Advances in neural information processing systems}, pp.\
  1281--1288, 2005.

\bibitem[Colas et~al.(2018{\natexlab{a}})Colas, Fournier, Sigaud, and
  Oudeyer]{colas2018curious}
Colas, C., Fournier, P., Sigaud, O., and Oudeyer, P.
\newblock {CURIOUS:} intrinsically motivated multi-task, multi-goal
  reinforcement learning.
\newblock \emph{CoRR}, abs/1810.06284, 2018{\natexlab{a}}.

\bibitem[Colas et~al.(2018{\natexlab{b}})Colas, Sigaud, and
  Oudeyer]{colas2018gep}
Colas, C., Sigaud, O., and Oudeyer, P.-Y.
\newblock Gep-pg: Decoupling exploration and exploitation in deep reinforcement
  learning algorithms.
\newblock \emph{International Conference on Machine Learning (ICML)},
  2018{\natexlab{b}}.

\bibitem[Eysenbach et~al.(2019)Eysenbach, Gupta, Ibarz, and
  Levine]{eysenbach2018diayn}
Eysenbach, B., Gupta, A., Ibarz, J., and Levine, S.
\newblock {Diversity is All You Need: Learning Skills without a Reward
  Function}.
\newblock In \emph{International Conference on Learning Representations
  (ICLR)}, 2019.

\bibitem[Florensa et~al.(2017)Florensa, Duan, and
  Abbeel]{florensa2017stochastic}
Florensa, C., Duan, Y., and Abbeel, P.
\newblock {Stochastic neural networks for hierarchical reinforcement learning}.
\newblock In \emph{International Conference on Learning Representations
  (ICLR)}, 2017.

\bibitem[Florensa et~al.(2018{\natexlab{a}})Florensa, Degrave, Heess,
  Springenberg, and Riedmiller]{florensa2018selfsupervised}
Florensa, C., Degrave, J., Heess, N., Springenberg, J.~T., and Riedmiller, M.
\newblock {Self-supervised Learning of Image Embedding for Continuous Control}.
\newblock In \emph{Workshop on Inference to Control at NeurIPS},
  2018{\natexlab{a}}.

\bibitem[Florensa et~al.(2018{\natexlab{b}})Florensa, Held, Geng, and
  Abbeel]{held2018goalgan}
Florensa, C., Held, D., Geng, X., and Abbeel, P.
\newblock {Automatic Goal Generation for Reinforcement Learning Agents}.
\newblock In \emph{International Conference on Machine Learning (ICML)},
  2018{\natexlab{b}}.

\bibitem[Fu et~al.(2017)Fu, Co-Reyes, and Levine]{fu2017ex2}
Fu, J., Co-Reyes, J.~D., and Levine, S.
\newblock {EX 2 : Exploration with Exemplar Models for Deep Reinforcement
  Learning}.
\newblock In \emph{Advances in Neural Information Processing Systems (NIPS)},
  2017.

\bibitem[Fujimoto et~al.(2018)Fujimoto, van Hoof, and Meger]{fujimoto2018td3}
Fujimoto, S., van Hoof, H., and Meger, D.
\newblock {Addressing Function Approximation Error in Actor-Critic Methods}.
\newblock In \emph{International Conference on Machine Learning (ICML)}, 2018.

\bibitem[Gupta et~al.(2018{\natexlab{a}})Gupta, Eysenbach, Finn, and
  Levine]{gupta2018unsupervised}
Gupta, A., Eysenbach, B., Finn, C., and Levine, S.
\newblock Unsupervised meta-learning for reinforcement learning.
\newblock \emph{CoRR}, abs:1806.04640, 2018{\natexlab{a}}.

\bibitem[Gupta et~al.(2018{\natexlab{b}})Gupta, Mendonca, Liu, Abbeel, and
  Levine]{gupta2018structuredexploration}
Gupta, A., Mendonca, R., Liu, Y., Abbeel, P., and Levine, S.
\newblock {Meta-Reinforcement Learning of Structured Exploration Strategies}.
\newblock In \emph{Advances in Neural Information Processing Systems (NIPS)},
  2018{\natexlab{b}}.

\bibitem[Haarnoja et~al.(2018)Haarnoja, Zhou, Hartikainen, Tucker, Ha, Tan,
  Kumar, Zhu, Gupta, Abbeel, and Levine]{haarnoja2018sacapp}
Haarnoja, T., Zhou, A., Hartikainen, K., Tucker, G., Ha, S., Tan, J., Kumar,
  V., Zhu, H., Gupta, A., Abbeel, P., and Levine, S.
\newblock Soft actor-critic algorithms and applications.
\newblock \emph{CoRR}, abs/1812.05905, 2018.

\bibitem[Hausman et~al.(2018)Hausman, Springenberg, Wang, Heess, and
  Riedmiller]{hausman2018skillembedding}
Hausman, K., Springenberg, J.~T., Wang, Z., Heess, N., and Riedmiller, M.
\newblock {Learning an Embedding Space for Transferable Robot Skills}.
\newblock In \emph{International Conference on Learning Representations
  (ICLR)}, pp.\  1--16, 2018.

\bibitem[Hazan et~al.(2019)Hazan, Kakade, Singh, and Soest]{hazan2019provably}
Hazan, E., Kakade, S.~M., Singh, K., and Soest, A.~V.
\newblock Provably efficient maximum entropy exploration.
\newblock \emph{International Conference on Machine Learning (ICML)}, 2019.

\bibitem[Higgins et~al.(2017)Higgins, Matthey, Pal, Burgess, Glorot, Botvinick,
  Mohamed, and Lerchner]{higgins2016beta}
Higgins, I., Matthey, L., Pal, A., Burgess, C., Glorot, X., Botvinick, M.,
  Mohamed, S., and Lerchner, A.
\newblock $\beta$-vae: Learning basic visual concepts with a constrained
  variational framework.
\newblock \emph{International Conference on Learning Representations (ICLR)},
  2017.

\bibitem[Kaelbling(1993)]{kaelbling1993goals}
Kaelbling, L.~P.
\newblock {Learning to achieve goals}.
\newblock In \emph{International Joint Conference on Artificial Intelligence
  (IJCAI)}, volume vol.2, pp.\  1094 -- 8, 1993.

\bibitem[Kalakrishnan et~al.(2011)Kalakrishnan, Righetti, Pastor, and
  Schaal]{kalakrishnan2011learning}
Kalakrishnan, M., Righetti, L., Pastor, P., and Schaal, S.
\newblock Learning force control policies for compliant manipulation.
\newblock In \emph{2011 IEEE/RSJ International Conference on Intelligent Robots
  and Systems}, pp.\  4639--4644. IEEE, 2011.

\bibitem[Lillicrap et~al.(2016)Lillicrap, Hunt, Pritzel, Heess, Erez, Tassa,
  Silver, and Wierstra]{lillicrap2015continuous}
Lillicrap, T.~P., Hunt, J.~J., Pritzel, A., Heess, N., Erez, T., Tassa, Y.,
  Silver, D., and Wierstra, D.
\newblock {Continuous control with deep reinforcement learning}.
\newblock In \emph{International Conference on Learning Representations
  (ICLR)}, 2016.
\newblock ISBN 0-7803-3213-X.
\newblock \doi{10.1613/jair.301}.

\bibitem[Lopes et~al.(2012)Lopes, Lang, Toussaint, and
  Oudeyer]{lopes2012exploration}
Lopes, M., Lang, T., Toussaint, M., and Oudeyer, P.-Y.
\newblock Exploration in model-based reinforcement learning by empirically
  estimating learning progress.
\newblock In \emph{Advances in Neural Information Processing Systems}, pp.\
  206--214, 2012.

\bibitem[Mohamed \& Rezende(2015)Mohamed and Rezende]{mohamed2015variational}
Mohamed, S. and Rezende, D.~J.
\newblock Variational information maximisation for intrinsically motivated
  reinforcement learning.
\newblock In \emph{Advances in neural information processing systems}, pp.\
  2125--2133, 2015.

\bibitem[Nachum et~al.(2018)Nachum, Brain, Gu, Lee, and Levine]{nachum2018hiro}
Nachum, O., Brain, G., Gu, S., Lee, H., and Levine, S.
\newblock {Data-Efficient Hierarchical Reinforcement Learning}.
\newblock In \emph{Advances in Neural Information Processing Systems
  (NeurIPS)}, 2018.

\bibitem[Nair et~al.(2018)Nair, Pong, Dalal, Bahl, Lin, and
  Levine]{nair2018rig}
Nair, A., Pong, V., Dalal, M., Bahl, S., Lin, S., and Levine, S.
\newblock {Visual Reinforcement Learning with Imagined Goals}.
\newblock In \emph{Advances in Neural Information Processing Systems
  (NeurIPS)}, 2018.

\bibitem[Nielsen \& Nock(2010)Nielsen and Nock]{nielsen2010entropies}
Nielsen, F. and Nock, R.
\newblock Entropies and cross-entropies of exponential families.
\newblock In \emph{Image Processing (ICIP), 2010 17th IEEE International
  Conference on}, pp.\  3621--3624. IEEE, 2010.

\bibitem[Ostrovski et~al.(2017)Ostrovski, Bellemare, Oord, and
  Munos]{ostrovski2017count}
Ostrovski, G., Bellemare, M.~G., Oord, A., and Munos, R.
\newblock Count-based exploration with neural density models.
\newblock In \emph{International Conference on Machine Learning (ICML)}, pp.\
  2721--2730, 2017.

\bibitem[Pathak et~al.(2017)Pathak, Agrawal, Efros, and
  Darrell]{pathak2017curiosity}
Pathak, D., Agrawal, P., Efros, A.~A., and Darrell, T.
\newblock {Curiosity-Driven Exploration by Self-Supervised Prediction}.
\newblock In \emph{International Conference on Machine Learning (ICML)}, pp.\
  488--489. IEEE, 2017.

\bibitem[P{\'{e}}r{\'{e}} et~al.(2018)P{\'{e}}r{\'{e}}, Forestier, Sigaud, and
  Oudeyer]{pere2018unsupervised}
P{\'{e}}r{\'{e}}, A., Forestier, S., Sigaud, O., and Oudeyer, P.-Y.
\newblock {Unsupervised Learning of Goal Spaces for Intrinsically Motivated
  Goal Exploration}.
\newblock In \emph{International Conference on Learning Representations
  (ICLR)}, 2018.

\bibitem[Pong et~al.(2018)Pong, Gu, Dalal, and Levine]{pong2018tdm}
Pong, V., Gu, S., Dalal, M., and Levine, S.
\newblock {Temporal Difference Models: Model-Free Deep RL For Model-Based
  Control}.
\newblock In \emph{International Conference on Learning Representations
  (ICLR)}, 2018.

\bibitem[Rubin(1988)]{rubin1988using}
Rubin, D.~B.
\newblock Using the sir algorithm to simulate posterior distributions.
\newblock \emph{Bayesian statistics}, 3:\penalty0 395--402, 1988.

\bibitem[Savinov et~al.(2018)Savinov, Raichuk, Marinier, Vincent, Pollefeys,
  Lillicrap, and Gelly]{savinov2018episodic}
Savinov, N., Raichuk, A., Marinier, R., Vincent, D., Pollefeys, M., Lillicrap,
  T., and Gelly, S.
\newblock Episodic curiosity through reachability.
\newblock \emph{arXiv preprint arXiv:1810.02274}, 2018.

\bibitem[Schaul et~al.(2015)Schaul, Horgan, Gregor, and Silver]{schaul2015uva}
Schaul, T., Horgan, D., Gregor, K., and Silver, D.
\newblock {Universal Value Function Approximators}.
\newblock In \emph{International Conference on Machine Learning (ICML)}, pp.\
  1312--1320, 2015.
\newblock ISBN 9781510810587.

\bibitem[Stadie et~al.(2016)Stadie, Levine, and Abbeel]{stadie2016exploration}
Stadie, B.~C., Levine, S., and Abbeel, P.
\newblock {Incentivizing Exploration In Reinforcement Learning With Deep
  Predictive Models}.
\newblock In \emph{International Conference on Learning Representations
  (ICLR)}, 2016.

\bibitem[Sutton et~al.(1999)Sutton, Precup, and Singh]{sutton1999between}
Sutton, R.~S., Precup, D., and Singh, S.
\newblock Between mdps and semi-mdps: A framework for temporal abstraction in
  reinforcement learning.
\newblock \emph{Artificial intelligence}, 112\penalty0 (1-2):\penalty0
  181--211, 1999.

\bibitem[Tang et~al.(2017)Tang, Houthooft, Foote, Stooke, Chen, Duan, Schulman,
  {De Turck}, and Abbeel]{tang2017hashtag}
Tang, H., Houthooft, R., Foote, D., Stooke, A., Chen, X., Duan, Y., Schulman,
  J., {De Turck}, F., and Abbeel, P.
\newblock {{\#}Exploration: A Study of Count-Based Exploration for Deep
  Reinforcement Learning}.
\newblock In \emph{Neural Information Processing Systems (NIPS)}, 2017.

\bibitem[Todorov et~al.(2012)Todorov, Erez, and Tassa]{todorov12mujoco}
Todorov, E., Erez, T., and Tassa, Y.
\newblock {MuJoCo: A physics engine for model-based control}.
\newblock In \emph{IEEE/RSJ International Conference on Intelligent Robots and
  Systems (IROS)}, pp.\  5026--5033, 2012.
\newblock ISBN 9781467317375.
\newblock \doi{10.1109/IROS.2012.6386109}.

\bibitem[Veeriah et~al.(2018)Veeriah, Oh, and Singh]{veeriah2018many}
Veeriah, V., Oh, J., and Singh, S.
\newblock Many-goals reinforcement learning.
\newblock \emph{arXiv preprint arXiv:1806.09605}, 2018.

\bibitem[Warde{-}Farley et~al.(2018)Warde{-}Farley, de~Wiele, Kulkarni,
  Ionescu, Hansen, and Mnih]{wardefarley2018discern}
Warde{-}Farley, D., de~Wiele, T.~V., Kulkarni, T., Ionescu, C., Hansen, S., and
  Mnih, V.
\newblock Unsupervised control through non-parametric discriminative rewards.
\newblock \emph{CoRR}, abs/1811.11359, 2018.

\bibitem[Zhao \& Tresp(2019)Zhao and Tresp]{zhao2019rankweight}
Zhao, R. and Tresp, V.
\newblock Curiosity-driven experience prioritization via density estimation.
\newblock \emph{CoRR}, abs/1902.08039, 2019.

\bibitem[Zhao et~al.(2019)Zhao, Sun, and Tresp]{zhao2019maximum}
Zhao, R., Sun, X., and Tresp, V.
\newblock Maximum entropy-regularized multi-goal reinforcement learning.
\newblock In \emph{International Conference on Machine Learning}, pp.\
  7553--7562, 2019.

\end{thebibliography}
